\documentclass[11pt]{article}
\usepackage[top=1in, left=1in, right=1in, bottom=1in, headheight=15pt,headsep=10pt]{geometry}

\usepackage{anyfontsize}
\usepackage{titlesec} 
\usepackage{titletoc} 
\usepackage{lmodern}
\usepackage{graphicx}
\usepackage{subfigure}
\usepackage{booktabs} 
\usepackage{fancyhdr}
\usepackage[skip=0pt]{subcaption}
\usepackage{diagbox}
\usepackage[round]{natbib}
\usepackage[backref=page]{hyperref}

\usepackage{amsmath}
\usepackage{amssymb}
\usepackage{mathtools}
\usepackage{amsthm}
\usepackage{mathrsfs}
\usepackage{appendix}
\usepackage{cancel}    
\usepackage{url}            
\usepackage{booktabs}       
\usepackage{microtype}      
\usepackage{setspace}
\usepackage{amssymb}
\usepackage{url}
\usepackage{lettrine}
\usepackage{framed}
\usepackage{color}
\usepackage{xfrac}
\usepackage{nicefrac}
\usepackage[svgnames, x11names]{xcolor}
\usepackage{algorithm}
\usepackage{algpseudocodex}
\usepackage{overpic}
\usepackage{dsfont}
\usepackage{booktabs} 
\usepackage{tikz}
\usepackage{wrapfig}
\usepackage{mathbbol}
\usepackage{multirow}
\usetikzlibrary{patterns}
\usetikzlibrary{decorations.pathreplacing}
\newlength{\bibitemsep}
\newlength{\bibparskip}
\let\oldthebibliography\thebibliography
\renewcommand\thebibliography[1]{%
  \oldthebibliography{#1}%
  \setlength{\parskip}{\bibitemsep}%
  \setlength{\itemsep}{\bibparskip}%
}
\usepackage{subfigure}
\usepackage{tabularx}

\usepackage{hyperref}
\hypersetup{
	colorlinks=true,
	linkcolor=Blue4,
	urlcolor=blue,
	citecolor=Blue4,
	linkbordercolor={0 0 1}
}
\usepackage[capitalize]{cleveref}
\crefname{ALC@unique}{Line}{lines}
\theoremstyle{plain}
\newtheorem{theorem}{Theorem}[section]
\newtheorem{proposition}{Proposition}[section]
\newtheorem{lemma}{Lemma}[section]

\newtheorem{example}{Example}

\theoremstyle{definition}
\newtheorem{definition}{Definition}
\newtheorem{assumption}{Assumption}
\newtheorem{claim}{Claim}

\theoremstyle{remark}

\usepackage{hyperref}
\usepackage{url}
\usepackage{booktabs}
\usepackage[justification=raggedright,skip=2pt]{caption}
\newcommand\numberthis{\addtocounter{equation}{1}\tag{\theequation}}
\newcommand{\ind}[1]{\mathbb{I}{\left[#1\right]}}

\newcommand{\argmax}{\mathop{\mathrm{argmax}}}

\def\R{\mathbb{R}}
\def\N{\mathbb{N}}
\def\P{\mathbb{P}}
\def\HP{\widehat{\mathbb{P}}}

\def\E{\mathbb{E}}
\def\Z{\mathbb{Z}}
\def\x{\mathbf{x}}

\def\B{\mathbf{B}}
\def\v{\mathbf{v}}
\def\W{\mathbf{W}}

\def\l{\ell}

\def\bidclass{\mathcal{B}}
\def\val{\mathsf{V}} 
\def\price{\mathsf{P}} 
\def\ut{\mathsf{U}}
\def\numbid{m}
\def\allbids{\boldsymbol{\beta}}
\def\ibid{\mathbf{b}}

\def\path{\mathfrak{p}}

\def\maxbid{M}

\def\safepolicy{\Pi}

\usepackage{tikz}
\usetikzlibrary{positioning, arrows.meta}

\newcommand{\DAG}[1]{\mathcal{G}^{#1}(\mathsf{N}_\v, \mathsf{E}_\v, \omega_\v)}

\newcommand{\optufclass}[1]{\mathcal{S}_{#1}^*}

\newcommand{\otherbid}[1]{\boldsymbol{\beta}_{-}^{#1}}
\newcommand{\ordotherbid}[2]{\boldsymbol{\beta}_{-, #1}^{-(#2)}}

\newcommand{\xDAG}[2]{\mathcal{G}^{#1}(\mathsf{N}_{#2}, \mathsf{E}_{#2}, \omega_{#2})}

\allowdisplaybreaks[4]

\NewDocumentCommand{\ARef}{ s s m }{%
    \IfBooleanTF{#2}{}{%
        \cref{#3}%
    }%
    \IfBooleanT{#1}{%
        \IfBooleanF{#2}{%
            , %
        }%
        Line~\ref{#3}%
    }%
}

\titleformat{\section}
  {\normalfont\bfseries\sffamily\large}
  {\thesection.}{0.5em}{}

\titleformat{\subsection}
  {\normalfont\bfseries\sffamily\normalsize}
  {\thesubsection.}{0.5em}{}

\titleformat{\subsubsection}
  {\normalfont\bfseries\sffamily\normalsize}
  {\thesubsubsection.}{0.5em}{}
  
\title{\vspace{-1cm}Learning to Bid in Discriminatory Auctions with Budget Constraints\footnote{Accepted at the 29th International Conference on Artificial Intelligence and Statistics (AISTATS) 2026.}}
\author{
  Negin Golrezaei\\
  \footnotesize Sloan School of Management, Massachusetts Institute of Technology, \href{mailto:golrezae@mit.edu}{\textsf{golrezae@mit.edu}}\\[1ex]
  Sourav Sahoo\\
  \footnotesize Operations Research Center, Massachusetts Institute of Technology, \href{mailto:sourav99@mit.edu}{\textsf{sourav99@mit.edu}}
}

\date{}
\begin{document}
\maketitle

\begin{abstract}
We study repeated bidding in multi-unit discriminatory (pay-as-bid) auctions for a single bidder with per-round utility equal to value minus $\alpha$ times payment, where $\alpha\in[0,1]$ is a cost-of-capital parameter. The bidder aims to maximize cumulative utility over $T$ rounds subject to a total budget $B$. The problem is challenging even without budgets: the action space is exponential in the bidder’s maximum demand \(\maxbid\), and the valuation vector (context) varies over time. Exploiting a decomposition of utility across units, we develop polynomial-time learning algorithms based on shortest paths in a directed acyclic graph, obtaining sublinear regret under both full-information and bandit feedback. In the bandit setting, the regret is independent of the number of contexts due to complete cross-learning: observing the utility of the chosen action under the realized context reveals the utility for the same action under all counterfactual contexts. With budget constraints, when the average normalized per-round budget \(\rho=\frac{B}{\maxbid T}<1\), we design a coupled primal-dual algorithm in which the DAG-based procedure uses dual-adjusted edge weights for primal updates, while online gradient descent updates the dual variable, yielding $\rho$-approximate sublinear regret. Finally, we give implementations whose per-round time and space are independent of the number of contexts, enabling scalability to large or even infinite context spaces.
\end{abstract}
\setstretch{1.33}
\vspace{-0.4cm}

\setcounter{tocdepth}{2}
\tableofcontents
\section{Introduction}

Multi-unit discriminatory auctions, also known as pay-as-bid (PAB) auctions, are commonly used in treasury auctions~\citep{brenner2009sovereign} and electricity markets~\citep{maurer2011electricity}. In this format, $K$ identical units of an item are sold to bidders who may demand multiple units, and each winning bidder pays exactly what they bid for each allocated unit. The well-known first-price auction~(FPA) is a special case of PAB auction when $K=1$.

Bidding in PAB auctions is notoriously challenging. Unlike truthful mechanisms such as VCG~\citep{vickrey1961counterspeculation, clarke1971multipart, groves1973incentives}, reporting one’s true valuation is not optimal in PAB auctions~\citep{ausubel2014demand}. Even under restrictive assumptions like common priors and equilibria play, computing optimal strategies is intractable due to the multi-dimensional nature of valuations~\citep{kasberger2025bidding}. The complexity is further amplified in repeated settings with budget constraints. Here, bidders must adapt to strategic or even adversarial competitors, time-varying valuations, navigate an exponentially large bid space (even with discrete bids, as is common in practice), and resolve the ``spend-or-save'' dilemma, where budget allocation today affects outcomes in the future. These challenges underscore the need for efficient and computationally tractable bidding strategies for budget-constrained bidders in repeated PAB auctions.

\subsection{Our Contributions}
We design bidding algorithms for a budget-constrained bidder in repeated multi-unit PAB auctions. 

\textbf{Modeling.} We consider the bidder’s utility as value minus $\alpha$ times payment, where $\alpha \in [0,1]$ is a cost-of-capital parameter~\citep{balseiro2019black}. This unifies two existing common behavioral models: traditional \emph{profit maximizers} $(\alpha=1)$ and \emph{value maximizers} $(\alpha=0)$~\citep{balseiro2021landscape,lv2023utility, deng2022fairness}. 
We model the budget as a \emph{hard constraint}, requiring it to be satisfied for every realized sequence of auctions, as is common in practice~\citep{wang2023learning,castiglioni2023online}. While prior works have focused on the setting where competing bids are drawn i.i.d. from a distribution~\citep{feng2023online,han2024optimal}, we allow the competing bids to be generated adversarially. The individual bids are integral multiples of $\epsilon \in (0,1)$, reflecting the minimum bid increments in real-world auctions such as eBay~\citep{ebay_automatic_bidding_2026} and U.S. Treasury auctions~\citep{hortaccsu2018bid}.

We focus on \emph{no-overbidding (NOB)} strategies, which require that cumulative bids never exceed cumulative valuations~(see \cref{def:overbid}). This ensures that the bidder’s utility is non-negative, regardless of the competing bids. For bidders with $\alpha=1$, prior work often imposes a stricter \emph{per-unit} NOB condition~\citep{markakis2015uniform,galgana2023learning}, which can be highly suboptimal when $\alpha<1$. Our cumulative notion of NOB is therefore more general and necessitates new techniques for designing bidding algorithms.

\textbf{No Budget Constraints.} We first study the offline problem of computing the hindsight-optimal strategy without budget constraints, assuming the competing bids are known in advance. The main challenge is that the number of NOB strategies grows as $O(\epsilon^{-\maxbid})$, where $\maxbid$ is the bidder’s maximum demand. To address this, we exploit that the bidder’s utility decomposes across units. Leveraging this structure, we construct an edge-weighted directed acyclic graph (DAG) of size $\text{poly}(\maxbid,1/\epsilon)$ in which each $s$–$d$ path corresponds to a NOB strategy and the associated path weight is an affine transformation of that strategy’s total utility. So, computing the optimal offline NOB strategy reduces to finding a shortest path in the DAG (\cref{thm:DAG-uf-policy}).

In the online setting, valuation vectors are drawn i.i.d.\ across rounds from a distribution $\mathcal D$ supported on $\mathcal V$. We first consider the case where $|\mathcal V|$ is finite~(and small), and later extend the results to large or infinite context spaces. Treating each valuation vector $\v\in[0,1]^{\maxbid}$ as a context, the bidder maps each context to a NOB strategy~(action) and, building on the offline reduction, maintains one DAG per context. The bidder then performs exponential-weights updates over $s$-$d$ paths using a variant of the weight-pushing method of \citet{takimoto2003path}. Under full-information feedback, all the competing bids are revealed \emph{ex post},\footnote{Strictly speaking, it suffices to reveal the top $K$ \textit{competing} bids (equivalently, the top $K+r_t$ bids overall), where $r_t$ is the number of units won by the bidder in round $t$.} allowing the bidder to update all DAGs. This algorithm does not require knowledge of $\mathcal D$ and achieves $\widetilde{O}(M^{3/2}\sqrt{T})$ regret, independent of $|\mathcal V|$ (\cref{thm:regret-safe}).\footnote{Here, $\widetilde{O}(\cdot)$ hides logarithmic factors in $1/\epsilon$.}$^,$\footnote{As $\v\in[0,1]^{\maxbid}$, the per-round utility is $O(\maxbid)$. Rescaling utilities by $1/\maxbid$ to lie in $[0,1]$ (as is standard in online learning) scales all stated regret bounds by $1/\maxbid$.}
 
In the bandit setting, only the bidder’s allocation is revealed after each auction. Under such limited feedback, a naive contextual bandit approach that updates only the DAG corresponding to the realized context typically incurs an unavoidable $O(\sqrt{|\mathcal V|})$ dependence in the regret bound. In contrast, in our setting the realized utility under the observed context helps infer counterfactual utilities that can be leveraged to update \emph{all} contexts, yielding \emph{complete cross learning} between contexts, a notion introduced by~\citep{balseiro2019contextual}. As a result, our regret bounds are independent of $|\mathcal V|$. Specifically, when the context distribution $\mathcal D$ is known, we exploit this knowledge to construct edge-weight estimators (cf.~\cref{eq:w-hat-estimate-exp3}), obtaining a regret bound of $\widetilde{O}(\maxbid^2\epsilon^{-3/2}\sqrt{T})$ (\cref{thm:regret-safe-bandit}). However, when $\mathcal D$ is unknown, constructing such estimators is more delicate. In this setting, we obtain a slightly weaker regret bound of $\widetilde{O}(\maxbid^2\epsilon^{-1}T^{2/3})$, while the learner overbids only on $\widetilde{O}(\maxbid\epsilon^{-1}T^{2/3})$ exploration rounds with high probability (\cref{thm:regret-safe-bandit-unk}). Finally, we establish a regret lower bound of $\Omega(\maxbid\sqrt{T})$~(\cref{thm:regret-LB}).

\textbf{With Budget Constraints.} Having addressed the unconstrained regime, we turn to the case where the bidder has a total budget $B=\rho\maxbid T$ over $T$ rounds. This introduces a long-term constraint that couples decisions across time. Under NOB and the payment rule, the per-round payment is at most $\maxbid$, so when $\rho\ge 1$ the budget never binds and the problem reduces to the unconstrained setting. The interesting regime is $\rho<1$, where the bidder must actively pace spending. 

Prior work on budget-constrained bidding has primarily focused on single-item auctions and has developed primal-dual based algorithms where the dual variable acts as a \emph{pacing multiplier} that directly scales values to set bids~\citep{balseiro2019learning,gaitonde2023budget,lucier2023autobidders}. In contrast, we develop a coupled primal-dual framework tailored to multi-unit PAB auctions. Building on \citet{castiglioni2022online}, we integrate the dual variable into our DAG-based primal regret minimizer via dual-adjusted edge weights (cf.~\cref{eq:edge-weight-online-PD}), while updating the dual variable using online gradient descent. The resulting algorithm jointly balances utility maximization and budget feasibility, yielding $\rho$-approximate regret of order $\widetilde{O}(R_T/\rho)$, where $R_T$ is the regret of the underlying primal algorithm.

\textbf{Efficient Implementation.} We next give efficient implementations of our learning algorithms for the full information setting and bandit setting with unknown context distribution that preserve the regret guarantees stated above while making the per-round time and space complexity \emph{independent} of the number of contexts $|\mathcal V|$. Thus, the algorithms remain tractable even for very large or infinite context spaces. The key observation is that edge utilities are affine in the context and decompose across edges, allowing us to maintain edge-specific coefficients shared across all contexts. Instead of a separate DAG for each $\v\in\mathcal V$, we use a \emph{super DAG} whose edge set contains the edges of all context-dependent DAGs and has the same order of size. We then update the shared coefficients on the super DAG each round, while keeping the core exponential-weights update via weight-pushing unchanged. The resulting per-round time and space complexity is $O(|\overline{\mathsf E}|)$, where $|\overline{\mathsf E}|=O(\maxbid^2/\epsilon^3)$ is the number of edges in the super DAG (see \cref{sec:infinite} for details).

\subsection{Related Works}

\textbf{Bidding under Constraints.} Bidding in repeated second-price auctions~(SPA) with budgets has been extensively studied in the stochastic setting~\citep{balseiro2015repeated, balseiro2019learning, balseiro2023the, feng2023online}, where the optimal $O(\sqrt{T})$ regret rate is achievable. These works leverage the truthfulness of SPA and adopt primal–dual frameworks in which the dual variables serve as \textit{pacing multipliers}. \citet{chen2024dynamic} studied an alternative budget-management approach known as \textit{throttling} and proposed a near-optimal throttling algorithm for SPA. For non-truthful auctions in the stochastic setting, \citet{gaitonde2023budget} designed a learning algorithm with $O(T^{3/4})$ \textit{pacing regret} relative to the optimal pacing multiplier, while \citet{aggarwal2025no} used the \textit{best Lipschitz bidding function} as a benchmark and achieved $O(\sqrt{T})$ regret. \citet{wang2023learning} proposed a primal–dual algorithm for FPA that attains $O(\sqrt{T})$ regret against the best budget-feasible strategy.

Under adversarially varying competing bids, \citet{balseiro2019learning} showed that no-regret learning is impossible even in SPA: no algorithm can guarantee a competitive ratio better than the average per-round budget, $\rho$. \citet{castiglioni2022online} studied FPA and obtained $\rho$-approximate sublinear regret. \citet{castiglioni2022unifying} extended this line of work by proposing a unified meta-algorithm that achieves \textit{best-of-both-worlds} regret guarantees—simultaneously robust to adversarial sequences and near-optimal in stochastic environments—for both FPA and SPA under budget and return-on-investment constraints. Finally, \citet{castiglioni2023online} showed that when $\rho$ is unknown, similar guarantees can be obtained by using weakly adaptive primal and dual regret minimizers, i.e., algorithms that ensure sublinear regret on every interval $[t_1,t_2]\subseteq[T]$. Our contribution is to extend this line of work to PAB auctions in the adversarial setting.




\textbf{Multi-unit Auctions.} Multi-unit auctions are a special class of combinatorial auctions in which identical units of a single good are allocated to the bidders. Examples of multi-unit auctions include emissions permit auctions \citep{cramton2002tradeable,alvarez2019assigning}, Treasury auctions \citep{nyborg2002bidder, garbade2005treasury, elsinger2019competition}, procurement auctions \citep{cramton2006dynamic}, and wholesale electricity markets \citep{tierney2008uniform, fabra2006designing}. The two most prominent payment rules in multi-unit auctions are the (a) \textit{uniform price} rule, where each bidder pays the same per-unit price~(usually the lowest accepted bid or the first rejected bid) and (b) \textit{discriminatory price} rule, where each bidder pays their bid for each unit they receive. While the revenue ranking between these formats is often ambiguous \citep{baisa2018large, ausubel2014demand}, uniform price auctions are favored for their perceived fairness (since all bidders pay the same per-unit price) and simplified bidding \citep{friedman1959}, whereas discriminatory price auctions are preferred for their transparency in payments. 

Recently, several works have studied learning to bid in repeated multi-unit auctions. For uniform price auctions,  \citet{branzei2023learning} study profit maximizers (\(\alpha=1\)) and obtain sublinear regret bounds; these were later improved in the bandit setting by \citet{potfer2024improved}. \citet{golrezaei2025learning} study the same auctions for value maximizers (\(\alpha=0\)) and give efficient learning algorithms with \(O(\sqrt{T})\) regret under both full-information and bandit feedback. For PAB auctions, \citet{galgana2023learning} provide polynomial-time learning algorithms for profit maximizers, which is the closest related work to ours. While our work builds on and generalizes theirs, there are substantive differences in the problem formulation, assumptions, algorithmic techniques, and online-learning guarantees. We summarize these differences in \cref{tab:gg24-vs-ours}. More recently, \citet{potfer2025comparing} compare the hardness of learning in uniform-price versus PAB auctions for profit maximizers under stochastic competing bids. Our work contributes to this line by studying PAB auctions for bidders with a general cost-of-capital parameter under budget constraints.



%

\begin{table}[!tbh]
\centering
\small
\caption{Comparison between \cite{galgana2023learning} and this paper.}
\begin{tabular}{p{0.22\linewidth} p{0.33\linewidth} p{0.37\linewidth}}
\toprule
\textbf{Facet} & \textbf{\cite{galgana2023learning}} & \textbf{This paper} \\
\midrule
Bidder's objective
& Profit-maximizing~($\alpha=1$) 
& Cost-of-capital parameter $\alpha\in[0,1]$\\
\addlinespace[2pt]
Notion of NOB 
& WLOG, Per-unit NOB ($b_j \le v_j$)
& Must enforce \emph{cumulative} NOB for general $\alpha$; per-unit NOB can lose an $O(\maxbid)$ factor when $\alpha<1$; see \cref{apx:example} \\
\addlinespace[2pt]
Offline optimization 
& Dynamic-programming graph based on per-unit NOB condition
&  Edge-weighted DAG whose nodes encode both bid values and the cumulative sum of bids along the path; edges enforce cumulative NOB. Naive extension of \cite{galgana2023learning} becomes shortest path in a DAG with multiple constraints, which is $\textsf{NP-Hard}$ in general~\citep{garey1979computers}.
\\
\addlinespace[2pt]
Budget constraints 
& --
& Considered \\
\addlinespace[2pt]
Context distribution 
& Known 
& Both known and unknown \\
\addlinespace[2pt]
Horizon knowledge ($T$) 
& Assumes known $T$ (or uses doubling trick for unknown $T$) 
& No prior horizon knowledge required when distribution is known. In the full information setting, this yields a \textit{best-of-both-worlds} behavior: constant regret in the stochastic setting while remaining adversarially robust\\
\bottomrule
\end{tabular}
\label{tab:gg24-vs-ours}
\end{table}

\textbf{Cross Learning in Bandits.} \citet{balseiro2019contextual} introduced \emph{cross learning} in contextual bandits, where the reward observed after playing an action in one context reveals (possibly partial) information about the reward that the same action would have obtained in other contexts. When such counterfactual rewards are revealed only for a subset of contexts, this is referred to as \emph{partial} cross learning; when they are revealed for all contexts, it is termed \emph{complete} cross learning. \citet{balseiro2019contextual} model this structure via a directed graph over contexts and derive sublinear regret guarantees under both stochastic and adversarial reward and context models, with rates governed by graph-dependent parameters. Subsequently, \citet{schneider2023optimal} studied the unknown context distribution setting with adversarial rewards and a polynomial number of arms, achieving an $O(\sqrt{T})$ regret bound (improving over the $O(T^{2/3})$ bound in \citet{balseiro2019contextual} for this regime). More recently, \citet{huang2025high} refined the analysis of the algorithm in \citet{schneider2023optimal} and obtained high-probability guarantees.

\textbf{Online Learning under Resource Constraints.} Bidding in repeated environments naturally fits within the bandits with knapsack (BwK) framework, a multi-armed bandit problem with global resource constraints~\citep{badanidiyuru2018bandits, immorlica2022adversarial, kesselheim2020online}. The BwK model has been generalized to handle arbitrary reward and resource functions~\citep{agrawal2019bandits}, contextual information~\citep{badanidiyuru2014resourceful, agrawal2016efficient}, and combinatorial semi-bandit feedback~\citep{sankararaman2018combinatorial}. In the stochastic setting, BwK admits vanishing $\widetilde{O}(\sqrt{T})$ regret, whereas in the adversarial case, no algorithm can achieve better than an $O(\log T)$ competitive ratio without additional assumptions. When the budget satisfies $B=\Omega(T)$, \citet{castiglioni2022online} obtain $\rho$-approximate sublinear regret, where $\rho=\frac{B}{T}$. A related benchmark, conceptually similar to uniform spending, is studied by \citet{braverman2025new}, who also establish sublinear regret guarantees. 

\section{Model}\label{sec:model}
\textbf{Notations.} For $n\in\N$, let $[n]=\{1,2,\dots,n\}$. Let $\ind{\cdot}$ denote the indicator function, i.e., $\ind{X}=1$ if the proposition $X$ is true and $0$ otherwise. We write $X\lesssim Y$ if $X\le CY$ for some absolute constant $C>0$, and define $X\gtrsim Y$ analogously. For any $\epsilon\in(0,1]$, let $\Z_\epsilon=\{k\epsilon: k\in\Z_{\ge 0}\}$. For a given integer $\maxbid$, define
\(
\bidclass := \left\{\ibid\in \Z_\epsilon^{\maxbid}\cap[0,1]^{\maxbid}:\; b_1\ge \cdots \ge b_{\maxbid}\right\},
\)
i.e., the set of all nonincreasing bid vectors on the $\epsilon$-grid.
\subsection{Pay-as-Bid Auctions}

In a PAB auction, there are $K$ identical units and bidders may demand multiple units. We study the bidding problem from the perspective of a single bidder whose maximum demand is $\maxbid\in[K]$. The bidder has a private valuation vector $\v\in[0,1]^{\maxbid}$ with diminishing marginal returns, i.e., $v_1\ge \cdots \ge v_{\maxbid}$~\citep{goldner2020reducing}. The valuation vector $\v$ is drawn from a distribution $\mathcal D$ supported on $\mathcal V$. For simplicity, we assume $\mathcal V$ is finite; in \cref{sec:infinite} we extend the results to infinite context spaces. For any $\v$, define the cumulative valuation vector $\W\in\R_+^{\maxbid}$ where $W_j=\sum_{\ell=1}^j v_\ell, \forall j\in[\maxbid]$.



\textbf{Allocation and Payment Rule.} The bidder submits a vector of bids $\ibid=[b_1, \dots, b_\maxbid]\in\bidclass$ {sorted in non-increasing order}. Each bid is an integral multiple of $\epsilon$, reflecting the discretization used in real-world auctions (e.g., increments of \$0.01). For $\ibid=[b_1, \dots, b_\maxbid]$, define $\B\in\bidclass$ as the cumulative bid vector where $B_j=\sum_{\l=1}^jb_\l, \forall j$. The competing bids are denoted by $\otherbid{}$ and the bid profile is $\boldsymbol{\beta}:=(\ibid; \otherbid{})$. 
The auctioneer elicits bids from all the bidders and the multiset of the top $K$ bids are called  \textit{winning} bids. Ties are resolved according to a fixed, publicly known deterministic rule, such as favoring lower indexed bidders~\citep{nisan2011best, chiesa2015knightian}. The bidder is allocated one unit for each of their bids in the multiset of winning bids. The total number of units allocated to the bidder under bid profile $\allbids$ is denoted by $x(\allbids)$. If the valuation vector is $\v=[v_1, \dots,v_\maxbid]$, the total value obtained by acquiring $x(\allbids)$ units is $\val_\v(\allbids):=W_{x(\allbids)}=\sum_{j\leq x(\allbids)}v_j$ and total payment is the sum of the accepted bids of the bidder, i.e., $\price(\allbids):=B_{x(\allbids)}=\sum_{j\leq x(\allbids)}b_j$. 

The bidder intends to maximize their quasilinear utility function $\ut_\v(\cdot)$. Formally, for a bid profile $\allbids=(\ibid; \otherbid{})$ and cost of capital $\alpha\in[0, 1]$,
\begin{align}\label{def:utility}
    \ut_\v(\allbids):=\val_\v(\allbids)-\alpha\price(\allbids)\,.
\end{align}
The utility function in \cref{def:utility} captures two well-studied bidder behavioral model. The traditional profit maximizer model is obtained when $\alpha=1$~\citep{borgers2015introduction}, while $\alpha=0$ corresponds to the value maximizer model, which has gained attention in the context of autobidders~\citep{balseiro2021robust,balseiro2021landscape,deng2022fairness}. 

\begin{definition}\label{def:overbid}
  For a given $\v$, $\ibid = [b_1, b_2, \ldots, b_{\maxbid}]$ is a \textit{no-overbidding~(NOB)} strategy if for all $\l \in [\maxbid]$, $B_\l \leq W_\l$, where $B_\l = \sum_{j=1}^\l b_j$ and $W_\l = \sum_{j=1}^\l v_j$. The collection of all NOB strategies corresponding to any $\v\in\mathcal{V}$ is $\bidclass_\v\subseteq\bidclass$.
\end{definition}

\begin{assumption}\label{assum:overbid}
    The bidder follows NOB strategies.
\end{assumption}
Our notion of NOB in \cref{assum:overbid} is standard in multi-unit auctions~\citep{christodoulou2016bayesian}. In prior work on profit-maximizing bidders, a more restrictive \emph{per-unit} NOB condition of $b_j \le v_j$ for all $j$ can be imposed without loss of generality~\citep{markakis2015uniform, galgana2023learning}. But this restriction can be highly suboptimal when $\alpha<1$ (see \cref{apx:example}). \cref{assum:overbid} ensures that $\ut_{\v}(\allbids)\geq 0$ for all $\otherbid{}$ and $\alpha \in [0,1]$.\footnote{If we normalize the utility of the outside option (not participating in the auction) to $0$,  this is referred to as the \textit{individual rationality} (IR) or \textit{participation} constraint. For a broad class of auctions, bidders adopt NOB strategies precisely to satisfy the IR constraint~\citep{bhawalkar2011welfare,de2013inefficiency}.} To see this, fix any $\v$, $\alpha \in [0,1]$, and competing bids $\otherbid{}$. For any NOB strategy $\ibid$, let $\allbids = (\ibid;\otherbid{})$. Then
$
\ut_\v(\allbids) = \val_\v(\allbids) - \alpha \price(\allbids) 
= W_{x(\allbids)} - \alpha B_{x(\allbids)} \geq0$.
Moreover, for any overbidding strategy, i.e., $\exists\l\in[\maxbid]$ such that $B_\l> W_\l$, there exists a competing bid profile for which $\ut_\v(\cdot)<0$ (see \cref{apx:overbidding-failure}). As competing bids can be adversarial, bidders follow NOB strategies to ensure that $\ut_\v(\cdot)\geq0$. 

\begin{example}
    Consider a PAB auction with $K=5$ identical units and $\epsilon=0.1$. The bidder's maximum demand is $\maxbid=5$,  and $\mathbf{v}=[1, 0.9, 0.7, 0.6, 0.4]$. Suppose their submitted strategy is $\ibid=[0.9, 0.9, 0.7, 0.6, 0.5]$. It can be verified that $B_j\leq W_j, \forall j\in[\maxbid]$ implying $\ibid$ is a NOB strategy. Suppose $\otherbid{}=[1, 0.8, 0.4, 0.3, 0.2]$. Thus, the winning bids are $[1, \underline{0.9}, \underline{0.9}, 0.8, \underline{0.7}]$ and the bidder is allocated $3$ units~(corresponding to the underlined bids). Hence, for any $\alpha\in[0, 1]$, $\val_\v(\allbids)=1+0.9+0.7=2.6$ and $\price(\allbids)=0.9+0.9+0.7=2.5$ which implies that $\ut_\v(\allbids)=2.6-2.5\alpha > 0$. 
\end{example}
\subsection{Problem Statement} 
We study the bidding problem for a single budget-constrained bidder in a repeated setting over $T$ rounds, where a PAB auction is conducted in each round. The bidder has a fixed budget $B$, and the total expenditure across the $T$ rounds must not exceed $B$. Following prior work on budget-constrained bidders, we assume that $B=\rho \maxbid T$ for some constant $\rho \geq 0$~\citep{balseiro2019learning,gaitonde2023budget,castiglioni2023online}. 

In round $t \in [T]$, the bidder first observes their valuation vector $\v^t \in \mathcal{V}$, sampled i.i.d.\ from the distribution $\mathcal{D}$, and then submits $\ibid^t$. If the competing bids are $\otherbid{t}$, the resulting bid profile is $\allbids^t = (\ibid^t; \otherbid{t})$. Given cost of capital $\alpha \in [0,1]$ and allocation $x(\allbids^t)$ units, the bidder’s value $\val_{\v^t}(\allbids^t)$, payment $\price(\allbids^t)$, and utility $\ut_{\v^t}(\allbids^t)$ are defined as before. After each auction, the bidder receives feedback and updates their bidding policy for future rounds. We consider two standard feedback models: in the \textit{full-information setting}, all competing bids $\otherbid{t}$ are revealed, whereas in the \textit{bandit setting}, only the allocation $x(\allbids^t)$ is observed.

\textbf{Baseline.} 
We compare the bidder’s performance against the optimal stationary policy subject to budget constraints. Let $\safepolicy$ be the class of policies that maps valuation vectors to bidding strategies:
\begin{align}\label{eq:def-safepolicy}
    \safepolicy = \Big\{\pi: \mathcal{V}\to\bigcup_{\v\in\mathcal{V}}\bidclass_\v, ~\text{s.t.}~ \pi(\v)\in\bidclass_\v\Big\}\,.
\end{align}
The baseline in this setting is the expected utility subject to budget constraints over $T$ rounds:
\begin{align}\label{eq:opt-sto}
\mathsf{OPT}&:=\max_{\pi\in\safepolicy\cup\bot}\sum_{t=1}^T\E\left[\ut_{\v^t}(\pi(\v^t); \otherbid{t})\right]\nonumber\\
    \text{s.t.}\quad&\sum_{t=1}^T\price(\pi(\v^t); \otherbid{t}) \leq \rho\maxbid T\,.\tag{\textsc{Opt}}
\end{align}
Here, $\bot$ is the policy that maps all valuation vectors to a null action $\emptyset$ which satisfies $\val(\emptyset; \otherbid{t})=\price(\emptyset; \otherbid{t})=0, \forall \otherbid{t}$. Recall that the competing bids are adversarially chosen. The baseline assumes prior knowledge of the competing bids and finds the optimal policy $\pi^*\in\safepolicy\cup\bot$, where the expectation is taken over the context distribution $\mathcal{D}$. This baseline is consistent with prior works on repeated auctions~\citep{balseiro2019contextual,schneider2023optimal,kumar2024strategically}.  

\textbf{Performance Metric.}   
We distinguish two regimes depending on $\rho$: (i) $\rho \ge 1$ and (ii) $\rho<1$. When $\rho \ge 1$, the budget constraint is automatically satisfied. This is because, under the NOB assumption and the auction’s payment rule, we have $\price(\allbids^t)\le \val_{\v^t}(\allbids^t)\le \maxbid$ for all $\otherbid{t}$. Hence, the total payment over $T$ rounds is at most $\maxbid T$ implying the budget constraint never binds. Equivalently, we can treat this regime as unconstrained. The benchmark is
\begin{align}\label{eq:opt-sto-NB}
\mathsf{OPT}_{nb}:=\max_{\pi \in \safepolicy}\sum_{t=1}^T \E\left[\ut_{\v^t}(\pi(\v^t);\otherbid{t})\right],\tag{\textsc{Opt-nb}}
\end{align}
and performance is measured via (standard) regret:
\[
\mathsf{Reg}_{nb}(T) := \mathsf{OPT}_{nb} - \sum_{t=1}^T \E\left[\ut_{\v^t}(\allbids^t)\right].
\]
When $\rho<1$, the problem resembles \emph{adversarial bandits with knapsacks} (BwK)~\citep{immorlica2022adversarial, kesselheim2020online}. In this regime, the standard choice of performance metric is \emph{$\rho$-approximate regret}:
\[
\rho\cdot\mathsf{Reg}(T) := \rho\cdot \mathsf{OPT} - \sum_{t=1}^T \E\left[\ut_{\v^t}(\allbids^t)\right].
\]
Our goal in the remainder of the paper is to design algorithms that guarantee sublinear ($\rho$-approximate) regret, i.e., $\mathsf{Reg}_{nb}(T)=o(T)$ and $\rho\cdot\mathsf{Reg}(T)=o(T)$.
\section{No Budget Constraint}\label{sec:learning-safe}
As stated earlier, we first consider the regime $\rho \ge 1$, which is equivalent to the setting without budget constraints. As will become clear in \cref{sec:with-budget}, the algorithms developed for this regime are a crucial building block for handling the case $\rho<1$. 
\vspace{-0.25cm}
\subsection{Offline Setting}
To design our online algorithms, we begin by solving the \emph{offline} optimization problem in \eqref{eq:opt-sto-NB}, which provides key structural insights. 


\begin{lemma}\label{lem:opt-policy}
For $\rho \ge 1$, an optimal stationary policy $\pi^*\in\Pi$ for \eqref{eq:opt-sto-NB} satisfies
\begin{align}\label{eq:offline-1}
 \pi^*(\v)\in\arg\max_{\ibid\in\bidclass_{\v}}\sum_{t=1}^T \ut_{\v}(\ibid;\otherbid{t}),
\qquad \forall \v\in\mathcal V.   
\end{align}
\end{lemma}

Motivated by \cref{lem:opt-policy}, we study the offline optimization problem above for a fixed valuation vector $\v=[v_1,\ldots,v_{\maxbid}]\in\mathcal V$, which will in turn guide the design of our online algorithms.

Let $\ordotherbid{t}{k}$ denote the $k^{\text{th}}$ smallest bid among the top $K$ competing bids in round $t$, i.e., among all bids excluding those of the bidder under consideration. If the bidder wins $r$ units in round $t$, then for each $s\in[r]$ it must hold that $b_s \ge \ordotherbid{t}{s}$~(assume the tie-breaking rule is incorporated in the indicator function in \cref{eq:utility-decomp}). Therefore, for any $\ibid=[b_1,\dots,b_{\maxbid}]\in\bidclass_\v$,
\begin{align}\label{eq:utility-decomp}
\ut_{\v}(\ibid;\otherbid{t})
= \sum_{j=1}^{\maxbid} (v_j-\alpha b_j)\cdot\ind{b_j \ge \ordotherbid{t}{j}} .
\end{align}
Summing \cref{eq:utility-decomp} over all $t\in[T]$ yields the objective in \cref{eq:offline-1}. Moreover, \cref{eq:utility-decomp} shows that this objective decomposes across units, a property we exploit to solve the offline problem efficiently.

\subsubsection{Constructing the DAG}
Since $|\bidclass_\v|=O(\epsilon^{-\maxbid})$, na\"ively enumerating all strategies is infeasible. Instead, for each context $\v\in\mathcal V$ we construct a context-dependent edge-weighted DAG of size $\text{poly}(\maxbid,1/\epsilon)$, which will be crucial for computing the optimal policy $\pi^*(\v)$ in \cref{lem:opt-policy}. In particular, computing $\pi^*(\v)$ is equivalent to finding a shortest (minimum-weight) $s$--$d$ path in this DAG. Finally, we introduce the notion of a \emph{super DAG}—a slight modification of the context-dependent DAGs—which will be useful for the analysis in \cref{ssec:online} and is central to our efficient implementations in \cref{sec:infinite}.

\textbf{Context-Dependent DAG.} Fix $\v\in\mathcal V$. The (context-dependent) DAG $\DAG{}$ consists of a source node $s$, a destination node $d$, and $\maxbid$ intermediate layers. Layer $\l$ corresponds to the $\l^{\text{th}}$ bid. 

\underline{Nodes.} A node in layer $\l$ is a triple $(\l,b_\l,s_\l)$, where $b_\l\in \Z_\epsilon\cap[0,1]$ and
\begin{align}\label{eq:S-UB}
s_\l\in \Z_\epsilon \quad\text{and}\quad s_\l \le W_\l,
\end{align}
recalling that $W_\l=\sum_{j=1}^\l v_j$ and $\Z_\epsilon$ denotes the set of nonnegative integral multiples of $\epsilon$. The constraint $s_\l\le W_\l$ enforces the NOB assumption. For convenience, define the source as $s=(0,\infty,0)$.

\underline{Edges and weights.} Edges exist only between consecutive layers. A directed edge from $(\l-1,b_{\l-1},s_{\l-1})$ to $(\l,b_\l,s_\l)$ exists if
\begin{align}\label{eq:edge-def-cond}
b_{\l-1}\ge b_\l \quad\text{and}\quad s_\l=s_{\l-1}+b_\l .
\end{align}
The weight edge $e=(\l-1,b_{\l-1},s_{\l-1})\to(\l,b_\l,s_\l)$ is
\begin{align}\label{eq:edge-weight-offline}
\omega_\v(e)
:= \sum_{t=1}^T \frac{1-(v_\l-\alpha b_\l)\cdot\ind{b_\l\ge \ordotherbid{t}{\l}}}{1+\alpha}.
\end{align}
All edges from nodes in layer $\maxbid$ to the destination node $d$ have weight $0$.

\underline{Size of the DAG.} In layer $\l$, we have $b_\l\in\Z_\epsilon\cap[0,1]$ and $s_\l\in\Z_\epsilon$ with $s_\l\le W_\l\le \l$. Thus, the number of possible values for $b_\l$ is $O(1/\epsilon)$ and for $s_\l$ is $O(\l/\epsilon)$, so layer $\l$ contains at most
$O\!\left(\frac{1}{\epsilon}\cdot\frac{\l}{\epsilon}\right)=O\!\left(\frac{\l}{\epsilon^2}\right)$ nodes. Summing over $\l\in[\maxbid]$ yields $|\mathsf N_\v| = O\left(\frac{\maxbid^2}{\epsilon^2}\right)$.

For edges, fix a node $(\l-1,b_{\l-1},s_{\l-1})$ in layer $\l-1$. Condition in \cref{eq:edge-def-cond} implies that the choice of $b_\l$ uniquely determines $s_\l$, so the out-degree is $O(1/\epsilon)$. Therefore, $|\mathsf E_\v| = O\!\left(|\mathsf N_\v|\cdot \frac{1}{\epsilon}\right)
= O\!\left(\frac{\maxbid^2}{\epsilon^3}\right)$, and hence the DAG has size $\text{poly}(\maxbid,1/\epsilon)$.

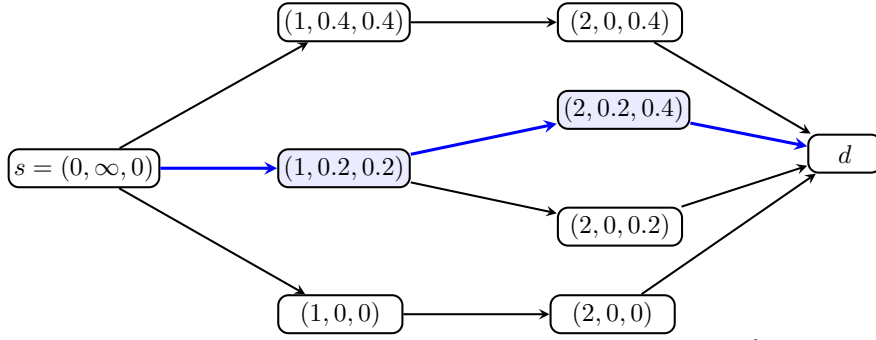
\begin{figure}[!tbh]
\centering
\resizebox{0.7\columnwidth}{!}{%
\begin{tikzpicture}[
  >=stealth, thick,
  every node/.style={draw, rounded corners, inner sep=2pt, font=\small, minimum width=17mm, minimum height=5.2mm},
  pathnode/.style={draw, rounded corners, inner sep=2pt, font=\small, minimum width=17mm, minimum height=5.2mm, fill=blue!8},
  pathedge/.style={->, very thick, blue}
]

\node[minimum width=18mm] (s) {$s=(0,\infty,0)$};
\node (L1a) [right=1.6cm of s, yshift=20mm]          {$(1,0.4,0.4)$};
\node[pathnode] (L1b) [right=1.6cm of s]             {$(1,0.2,0.2)$};
\node (L1c) [right=1.6cm of s, yshift=-20mm]         {$(1,0,0)$};

\node (L2a) [right=2.0cm of L1a]                     {$(2,0,0.4)$};
\node[pathnode] (L2b) [right=2.0cm of L1b, yshift=8mm]  {$(2,0.2,0.4)$};
\node (L2c) [right=2.0cm of L1b, yshift=-8mm]        {$(2,0,0.2)$};
\node (L2d) [right=2.0cm of L1c]                     {$(2,0,0)$};

\node[minimum width=10mm] (d) [right=1.6cm of L2b, yshift=-6mm] {$d$};
\draw[->]       (s) -- (L1a);
\draw[pathedge] (s) -- (L1b);
\draw[->]       (s) -- (L1c);

\draw[->]       (L1a) -- (L2a);
\draw[pathedge] (L1b) -- (L2b);   
\draw[->]       (L1b) -- (L2c);
\draw[->]       (L1c) -- (L2d);

\draw[->]       (L2a) -- (d);
\draw[pathedge] (L2b) -- (d);     
\draw[->]       (L2c) -- (d);
\draw[->]       (L2d) -- (d);

\end{tikzpicture}%
}
\caption{DAG for $M=2$ with $\epsilon=0.2$. Nodes are $(\l, b_\l, s_\l)$ with $s_\l=\sum_{j=1}^{\l} b_j \le W_\l$, where $W_1=0.5$ and $W_2=0.55$. Each layer $\l$ chooses one bid and enforces monotonicity ($b_{\l-1}\!\ge b_\l$) and NOB condition ($s_\l\!\le W_\l$). The blue path corresponds to $\ibid=[0.2,0.2]$.}
\label{fig:dag-two-layer}
\end{figure}





    


We now state the main result of this section, which establishes a one-to-one correspondence between $s$--$d$ paths in the context-dependent DAG $\DAG{}$ and NOB strategies in $\bidclass_\v$.

\begin{theorem}\label{thm:DAG-uf-policy}
There is a bijection between $s$--$d$ paths in $\DAG{}$ and strategies in $\bidclass_\v$.

\smallskip
\noindent\emph{\textbf{(Path $\leftrightarrow$ strategy).}}
A path $\path = s \to (1,b_1,s_1)\to \cdots \to (\maxbid,b_{\maxbid},s_{\maxbid})\to d$ corresponds to the strategy $\ibid=[b_1,\dots,b_{\maxbid}]\in\bidclass_\v$. Conversely, any strategy
$\ibid=[b_1,\dots,b_{\maxbid}]\in\bidclass_\v$ corresponds to the unique path $\path$ stated earlier with $s_j=\sum_{\l=1}^j b_\l$ for all $j\in[\maxbid]$.

\smallskip
\noindent\emph{\textbf{(Path weight).}}
Let $\path$ be the path corresponding to $\ibid=[b_1,\dots,b_{\maxbid}]$. Then
\[
\omega_\v(\path) := \sum_{e\in\path}\omega_\v(e)
= \frac{\maxbid T - \sum_{t=1}^T \ut_{\v}(\ibid;\otherbid{t})}{1+\alpha}.
\]
In particular, maximizing $\sum_{t=1}^T \ut_{\v}(\ibid;\otherbid{t})$ over $\ibid\in\bidclass_\v$
is equivalent to finding a shortest (minimum-weight) $s$-$d$ path in $\DAG{}$. Moreover, since $|\mathsf N_\v|\lesssim |\mathsf E_\v|$, a shortest path in $\DAG{}$ can be computed in
$O(|\mathsf E_\v|)=O(\maxbid^2/\epsilon^3)$ time and space.
\end{theorem}

\noindent\textbf{Super DAG.} The \emph{super DAG} $\mathcal G(\overline{\mathsf N},\overline{\mathsf E})$ has the same layered structure as the context-dependent DAG $\DAG{}$, but is independent of the context. In $\DAG{}$, a node $(\l,b_\l,s_\l)$ must satisfy $s_\l\in\Z_\epsilon$ and $s_\l\le W_\l$, where $W_\l=\sum_{j=1}^\l v_j$. In the super DAG, we replace this with the context-free constraint
\begin{align}\label{eq:S-UB-alt}
s_\l\in \Z_\epsilon \quad\text{and}\quad s_\l \le \l .
\end{align}
Because $W_\l\le \l$ for all $\v\in\mathcal V$, every node and edge feasible in $\DAG{}$ is also feasible in $\mathcal G(\overline{\mathsf N},\overline{\mathsf E})$. So, for every $\v\in\mathcal V$, we have $\mathsf N_\v \subseteq \overline{\mathsf N}$ and $\mathsf E_\v \subseteq \overline{\mathsf E}$. Moreover, for any $\v\in\mathcal{V}$, $\mathcal G(\overline{\mathsf N},\overline{\mathsf E})$ has the same order of size as $\DAG{}$: $|\overline{\mathsf N}|=O(\maxbid^2/\epsilon^2)$ and $|\overline{\mathsf E}|=O(\maxbid^2/\epsilon^3)$.

\subsection{Online Setting}\label{ssec:online}
We now build upon \cref{thm:DAG-uf-policy} to design a no-regret algorithm for the online setting, where a PAB auction is conducted in each round $t\in[T]$. 

\subsubsection{Full Information Feedback}
In the full-information setting, after each round \(t\) the bidder observes the competing bid profile \(\otherbid{t}\). In fact, it suffices to reveal only the top \(K\) competing bids, equivalently the top \(K+r_t\) bids in \(\allbids^t\), where \(r_t\) is the number of units won by the bidder in round \(t\). We model the problem as contextual online learning: each valuation vector \(\v\in\mathcal V\) is a context, and upon observing \(\v\), the bidder selects a strategy in \(\bidclass_{\v}\). By \cref{thm:DAG-uf-policy}, there is a bijection between strategies in \(\bidclass_{\v}\) and \(s\)–\(d\) paths in \(\DAG{}\). A natural idea is therefore to treat each path as an expert and run a no-regret algorithm such as Hedge~\citep{freund1997decision}.\footnote{This is referred to as \textit{Expanded Hedge}~\citep{koolen2010hedging} and as expanded exponential weights (EXP2) in \citet{audibert2014regret}.} However, a naive implementation is computationally infeasible, since it requires tracking \(O(\epsilon^{-\maxbid})\) experts. To circumvent this, we exploit the combinatorial structure of the action space and use a dynamic-programming-based variant of the weight-pushing method of \citet{takimoto2003path}, which maintains weights on edges rather than entire paths.

\textbf{Overview of the Algorithm.} The bidder maintains a context-dependent DAG $\DAG{t}$ for each $\v\in\mathcal V$. At the beginning of round $t$, the bidder observes $\v^t\sim\mathcal D$ and updates the edge probabilities $\{\phi_\v^t(e)\}_{e\in\mathsf E_\v}$ for \emph{all} $\v\in\mathcal V$ based on the previous-round values $\{\phi_\v^{t-1}(e)\}_{e\in\mathsf E_\v}$. Let $\omega_\v^t(e)$ denote the weight of edge $e\in\mathsf E_\v$ in round $t$ (for any $\v\in\mathcal V$), and let $\eta_t$ be the learning rate. Set $\eta_0=1$ and, for all $t\ge 1$, $\gamma_t := \frac{\eta_t}{\eta_{t-1}}$.

To compute $\phi_\v^t(\cdot)$, set $\Gamma_\v^{t-1}(d)=1$ and compute $\Gamma_\v^{t-1}(\cdot)$ bottom-up as follows. For every $u\in\mathsf N_\v$,
\begin{align}\label{eq:gamma-update}
   \Gamma^{t-1}_\v(u)=\sum_{v:u\to v\in\textsf{E}_\v}\Big(&\Gamma^{t-1}_\v(v)\cdot[\phi^{t-1}_\v(u\to v)]^{\gamma_t}\cdot\exp(-\eta_t \,\omega^{t-1}_\v(u \to v))\Big)\,.
\end{align}
Then, for every edge $e=u\to v \in \textsf{E}_\v$, update
\begin{align}\label{eq:phi-update}
    \phi^t_\v(e) = [\phi^{t-1}_\v(e)]^{\gamma_t}\cdot \exp(-\eta_t \,\omega^{t-1}_\v(e)) \cdot \frac{\Gamma^{t-1}_\v(v)}{\Gamma^{t-1}_\v(u)}.
\end{align}
As we show in the proof of \cref{thm:regret-safe}, the updates in \cref{eq:gamma-update,eq:phi-update} recovers the standard Hedge algorithm by setting $\eta_t=\eta$ for all $t\in[T]$. When $\eta_t$ time-decays, the update rule corresponds to the Decreasing Hedge algorithm~\citep{mourtada2019optimality}.

From \cref{eq:gamma-update} and \cref{eq:phi-update}, it is easy to verify that the edge probabilities \(\phi^t_\v(u \to \cdot)\) form a valid distribution over the out-neighbors of $u$. This motivates sampling edges sequentially in a Markovian fashion over \(\maxbid+1\) steps. Starting at node $s$, select an outgoing edge $s \to u$ with probability $\phi^t_{\v^t}(s \to u)$ and transition to node $u$, repeating the process until the destination $d$ is reached. The $s$-$d$ path obtained in this manner maps to a strategy in $\bidclass_{\v^t}$ as specified in \cref{thm:DAG-uf-policy}, which the bidder then submits. After the auction concludes, the bidder observes $\otherbid{t}$ and sets the edge weights in all the DAGs. Specifically, for each $\v\in\mathcal{V}$, and edge $\mathsf{E}_\v\ni e=(\l-1, b_{\l-1}, s_{\l-1})\to(\l, b_{\l}, s_{\l})$ in layer $\l\in[\maxbid]$, the edge weight is set as
\begin{align}\label{eq:edge-weight-online}
    \omega^t_\v(e) = \frac{1-(v_\l-\alpha b_\l)\cdot\ind{b_\l\geq \ordotherbid{t}{\l}}}{1+\alpha}\,.
\end{align}
All edges from nodes in layer $\maxbid$ to the destination node $d$ have weight $0$. Note that in this setting, the bidder does not require knowledge of the context distribution $\mathcal{D}$, and thus we may assume that $\mathcal{D}$ is unknown. The learning algorithm is formally presented in \cref{alg:weight-pushing-stoc}. 



\begin{theorem}\label{thm:regret-safe}
In the full-information setting without budget constraint, \cref{alg:weight-pushing-stoc} runs in $O(|\mathcal{V}|\maxbid^2/\epsilon^3)$ space and time per round. For $\epsilon\in(0,1)$ and any non-increasing sequence $\{\eta_t\}_{t=1}^T$ with $\eta_t>0$,
\[
    \mathsf{Reg}_{nb}(T) \lesssim \frac{\maxbid\log(1/\epsilon)}{\eta_T} + \maxbid^2 \sum_{t=1}^T \eta_t \,.
\]
If $\eta_t = \sqrt{\tfrac{\log 1/\epsilon}{\maxbid t}}, \forall t$, then 
$\mathsf{Reg}_{nb}(T)\lesssim\maxbid^{3/2}\sqrt{T\log 1/\epsilon}$. Alternatively, if $\eta_t = \eta= \sqrt{\tfrac{\log 1/\epsilon}{\maxbid T}}, \forall t$, the same regret bound (up to constants) is also achieved.\footnote{The variant with $\eta_t\propto t^{-1/2}$ is called an \textit{anytime} algorithm~\citep{lattimore2020bandit}, since it attains the optimal regret bound with no prior knowledge of $T$.}
\end{theorem}

\begin{algorithm}[!tbh]
\caption{\small No Budget Constraints~(Full Information)}
\label{alg:weight-pushing-stoc}
\small{
\begin{algorithmic}[1]
\Require Set of valuation vectors, $\mathcal{V}$, learning rates $\eta_t>0, \forall t\geq1$. Define $\eta_0=1, \phi^0_\v(e)=1$ and $\omega^0_\v(e)=0, \forall e\in \mathsf{E}_\v, \forall \v\in\mathcal{V}$.
\For{$t = 1, 2, \dots $}
    \State Observe an i.i.d. valuation vector sample $\v^t\sim\mathcal{D}$.
    \State Construct $\DAG{t}$ without weights $\forall\v\in\mathcal{V}$.\label{line:map-start}
    
    \For{$\v\in\mathcal{V}$}
    \State Obtain edge probabilities $\phi^t_\v(\cdot)$ following \cref{eq:gamma-update} and \cref{eq:phi-update}.
\EndFor
    \State Define initial node $u=s$ and path $\path^t=s$. 
    \While{$u\neq d$}
    \State Sample $v$ with probability $\phi^t_{\v^t}(u\to v)$.
    \State Append $v$ to the path $\path^t$; set $u\gets v$.
    \EndWhile
    \State Map $\path^t=s\to (1, b_1, s_1)\to\dots\to (\maxbid, b_\maxbid, s_\maxbid)\to d$, and submit $\ibid^t=[b_1, \dots, b_\maxbid]$.\label{line:map}
    
    \State Set edge weights per \cref{eq:edge-weight-online} for all $\v\in\mathcal{V}$.\label{line:set-wt}
\EndFor
\end{algorithmic}}
\end{algorithm}

\vspace{-0.1cm}
\subsubsection{Bandit Feedback}\label{sssec:bandit}
\textbf{Known Context Distribution.} In the bandit setting, the bidder observes the context $\v^t$, selects a bid $\ibid^t\in\bidclass_{\v^t}$, and after the auction observes only the allocation $x(\allbids^t)$. A natural approach is to treat this as a contextual bandit problem: maintain a context-dependent DAG for each $\v\in\mathcal V$ and update only the DAG corresponding to the realized context $\v^t$ using unbiased edge-weight estimators $\widehat{\omega}_\v^t(e)$.\footnote{This is referred to as the $S$-EXP3 algorithm~\citep{bubeck2012regret}.} However, this incurs an undesirable $O(\sqrt{|\mathcal V|})$ multiplicative dependence in the regret bound. We next exploit the structure of the utility function to obtain regret guarantees independent of $|\mathcal V|$.

We observe that the utility of a submitted strategy $\ibid$ reveals its counterfactual utility under every other context $\v'\in\mathcal V$. Specifically,
\[
\ut_{\v'}(\ibid; \otherbid{t})
=
\ut_{\v}(\ibid; \otherbid{t})
+\sum_{j=1}^{\maxbid}(v_j'-v_j)\cdot \ind{b_j\geq \ordotherbid{t}{j}} .
\]
All quantities on the right-hand side are observable whenever $\ibid$ is submitted in round $t$, even under bandit feedback. Thus, submitting $\ibid$ reveals its utility under \textit{all} contexts, yielding \emph{complete cross-learning} \citep{balseiro2019contextual}. This, in turn, allows us to construct edge-weight estimators for all context-dependent DAGs. Accordingly, the bidder runs \cref{alg:weight-pushing-stoc}, replacing $\omega^t_\v(e)$ by its estimator $\widehat{\omega}^t_\v(e)$ as in \cref{eq:w-hat-estimate-exp3}.

In round $t$, suppose the bidder observes $\v^t$ and submits strategy $\ibid^t\in\bidclass_{\v^t}$. If the corresponding path is $\path^t$ (see \ARef*{line:map}), then for any $\v\in\mathcal{V}$ and $e\in \mathsf{E}_\v$,
\begin{align}\label{eq:w-hat-estimate-exp3}
    \widehat{\omega}^t_\v(e) &=
        \frac{\omega^t_\v(e)}{q^t(e)}\cdot\ind{e\in\path^t},
\end{align}
where $\omega^t_\v(e)$ is defined in \cref{eq:edge-weight-online}, and
\begin{align}\label{eq:def-q}
    q^t(e)=\sum_{\v\in\mathcal{V}}\P[\v^t=\v]\sum_{\path\in\mathscr{P}_\v:e\in\path}\P[\path^t=\path|\v^t=\v]\,.
\end{align}
Here, $\mathscr P_\v$ denotes the set of $s$--$d$ paths in $\DAG{}$, and $q^t(e)$ is the unconditional probability that edge $e$ is selected in round $t$, averaged over $\v\sim\mathcal D$. Thus, computing \cref{eq:w-hat-estimate-exp3} requires knowledge of the context distribution $\mathcal D$. The estimator resembles the standard importance-weighted bandit estimator, but with edge-marginal probability $q^t(e)$ in the denominator. The main computational bottleneck is therefore evaluating $q^t(e)$, which we show can be done in $O(|\mathcal V|M^2/\epsilon^3)$ time and space (see \cref{apx:thm:regret-safe-bandit}). 


\begin{theorem}\label{thm:regret-safe-bandit} In the bandit setting under known context distribution and no budget constraint, \cref{alg:weight-pushing-stoc} with the estimator in \cref{eq:w-hat-estimate-exp3} runs in $O(|\mathcal{V}|\maxbid^2/\epsilon^3)$ space and time per round. For any $\epsilon\in(0, 1)$ and non-increasing sequence $\eta_t>0$,
\[
    \mathsf{Reg}_{nb}(T)\lesssim\frac{\maxbid \log 1/\epsilon}{\eta_T} + \frac{\maxbid^3}{\epsilon^3}\sum_{t=1}^T\eta_t\,.
\]
If $\eta_t = \frac{1}{\maxbid}\sqrt{\frac{\epsilon^3\log 1/\epsilon}{t}}, \forall t$, $\mathsf{Reg}_{nb}(T)\lesssim \frac{\maxbid^2}{\epsilon^{3/2}}\sqrt{T\log 1/\epsilon}$. Alternatively, if $\eta_t = \eta = \frac{1}{\maxbid}\sqrt{\frac{\epsilon^3\log 1/\epsilon}{T}}, \forall t$, the same regret bound (up to constants) is also achieved.
\end{theorem}

\textbf{Unknown Context Distribution.} Recall that knowledge of the distribution $\mathcal D$ was crucial for constructing the unbiased estimator $\widehat{\omega}^t_\v(e)$ in \cref{eq:w-hat-estimate-exp3}. For unknown $\mathcal D$, several changes are necessary.

\noindent\underline{Edge weights.}
For the regret analysis, we now maintain \emph{gains} instead of losses. In round $t$, for $\v\in\mathcal V$ and edge $e=(\ell-1,b_{\ell-1},s_{\ell-1})\to(\ell,b_\ell,s_\ell)$ in layer $\ell\in[\maxbid]$, 
\begin{align}\label{eq:edge-weight-online-unk}
    \omega^t_\v(e)
    :=
    \frac{\alpha+(v_\ell-\alpha b_\ell)\ind{b_\ell\geq \ordotherbid{t}{\ell}}}{1+\alpha}.
\end{align}
All edges from layer $\maxbid$ to $d$ have weight $0$. Since $v_\ell,b_\ell\in[0,1]$, $\omega^t_\v(e)\in[0,1]$.

\noindent\underline{Edge path cover and estimator.}
If the bidder samples only from the context-dependent DAG for the realized context $\v'$, then $p^t_{\v'}(e)=0$ for $e\in\mathsf E_\v\setminus\mathsf E_{\v'}$, so the estimator $\frac{\omega^t_{\v}(e)\cdot\ind{e\in\path^t}}{p^t_{\v'}(e)}$ is not well-defined. We therefore mix the exponential-weights distribution over paths in the DAG corresponding to $\v'$ with uniform exploration over an \emph{edge path cover} $\mathcal C$ of the super DAG~\citep{gyorgy2007line}: a collection of $s$-$d$ paths such that every edge $e\in\overline{\mathsf E}$ lies on some $\path\in\mathcal C$. Such a cover is obtained by fixing one $s$-$d$ path through each edge, so $|\mathcal C|\le |\overline{\mathsf E}|$.

After observing $\v^t=\v'$, the bidder draws uniformly from $\mathcal C$ with probability $\delta\in(0,1]$, and otherwise draws $\path\in\mathscr P_{\v'}$ with probability \(\propto \exp\left(\eta_t\sum_{s=1}^{t-1}\widehat{\omega}^s_{\v'}(\path)\right).\) Let $p^t_{\v'}(e)$ be the resulting conditional probability of selecting edge $e$. Then $p^t_{\v'}(e)\ge \delta/|\overline{\mathsf E}|>0$ for all $e\in\overline{\mathsf E}$. 

If $\v^t=\v'$, then for any $\v\in\mathcal V$ and $e\in\overline{\mathsf E}$, define
\begin{align}\label{eq:w-hat-estimate-exp3-unk}
    \widehat{\omega}^t_\v(e)
    =
    \frac{\omega^t_\v(e)}{p^t_{\v'}(e)}
    \ind{e\in\path^t}.
\end{align}
Thus, the bidder plays a NOB strategy in each round with probability at least $1-\delta$; choosing $\delta$ appropriately makes the number of overbidding rounds sublinear with high probability. The full procedure is given in \cref{alg:weight-pushing-bandit-unk}.

\begin{theorem}\label{thm:regret-safe-bandit-unk} In the bandit setting under unknown context distribution and no budget constraints, \cref{alg:weight-pushing-bandit-unk} runs in $O(|\mathcal{V}|\maxbid^2/\epsilon^3)$ space and time per round. For any $\epsilon\in(0, 1),$ using $\delta\in(0, 1]$ and non-increasing sequences $\eta_t\leq \frac{\delta}{\maxbid |\overline{\mathsf{E}}|}$, we get
\[
\mathsf{Reg}_{nb}(T) \lesssim\frac{\maxbid\log 1/\epsilon}{\eta_T}+ \frac{\maxbid^2|\overline{\mathsf{E}}|}{\delta}\sum_{t=1}^T\eta_t+\maxbid T\delta\,.
\]
Second, for any $\zeta\in(0, 1)$, with probability at least $1-\zeta$, the number of rounds in which the learner may submit an overbidding strategy is
\[
    \#\{t\in[T]:\ibid^t\notin\mathcal B_{\v^t}\}
    \le
    \delta T+\sqrt{\frac{T\log(1/\zeta)}{2}}.
\] 
If $\delta=\min\left(1, \left(\frac{\maxbid |\overline{\mathsf{E}}|\log 1/\epsilon}{T}\right)^{1/3}\right)$, and $\eta_t=\frac{\delta^2}{\maxbid|\overline{\mathsf{E}}|}$,  
\[
   \mathsf{Reg}_{nb}(T) \lesssim  \frac{\maxbid^2 T^{2/3}(\log 1/\epsilon)^{1/3}}{\epsilon} + \frac{\maxbid^4\log 1/\epsilon}{\epsilon^3}.
\]
Taking $\zeta=1/T$, with probability at least $1-1/T$,
\begin{align*}
  \#\{t\in[T]:\ibid^t\notin\mathcal B_{\v^t}\}    &\le\frac{\maxbid(\log(1/\epsilon))^{1/3}T^{2/3}}{\epsilon} + \sqrt{T\log T}.  
\end{align*}
\end{theorem}
\noindent\textbf{Lower Bound.} We conclude this section by establishing a regret lower bound showing that our learning algorithms in the full information setting and bandit setting with known context distribution are optimal in their dependence on $T$ (up to logarithmic factors).

\begin{theorem}\label{thm:regret-LB}
Fix any  $\alpha\in(\tfrac12,1]$. Then, for any learning algorithm, there exists a sequence of competing bids
$[\otherbid{t}]_{t\in[T]}$ such that, in the full information setting without budget constraints,
\(
\E[\mathsf{Reg}_{nb}(T)] = \Omega(\maxbid\sqrt{T}).
\)
Consequently, the same lower bound holds in the bandit setting.
\end{theorem}
\section{With Budget Constraint}\label{sec:with-budget}
Having developed polynomial-time learning algorithms for the unconstrained regime, we now turn to the budget-constrained case, i.e., $\rho<1$. Since the bids can be adversarial, the performance metric in this setting is \emph{$\rho$-approximate regret}: $\rho\cdot\mathsf{Reg}(T)= \rho\cdot\mathsf{OPT} - \sum_{t=1}^T\E[\ut_{\v^t}(\allbids^t)]$, where $\mathsf{OPT}$ is defined in \eqref{eq:opt-sto}. 

Primal–dual algorithms are commonly used to \emph{pace} bidders’ spending under budget constraints: in each round, the primal step produces a bid by scaling the bidder’s value using the current dual variable (a pacing multiplier), and the dual variable is updated via a gradient-based method~\citep{balseiro2019learning, gaitonde2023budget, lucier2023autobidders}. However, in multi-unit PAB auctions the optimal bid can be a nontrivial function of the entire valuation vector, so this simple ``scale-the-value'' primal update is no longer appropriate. 

\textbf{Overview of the Algorithm.} For ease of exposition, we focus on the full-information setting; the bandit case is deferred to \cref{apx:thm:regret-with-budgets}. For round $t$, consider the Lagrangian relaxation: $\mathcal L_t(\ibid,\lambda_t)=\ut_{\v^t}(\ibid;\otherbid{t})+\lambda_t(\rho\maxbid-\price(\ibid;\otherbid{t}))$ with $\lambda_t\geq0$. We consider a primal-dual based algorithm consisting of two subroutines: a primal regret minimizer~(RM) that aims to maximize $\mathcal L_t(\cdot,\lambda_t)$ and a dual RM that aims to minimize $\mathcal L_t(\ibid,\cdot)$. 

\underline{Primal RM.} We use \cref{alg:weight-pushing-stoc} as the primal RM with two changes. First, we replace the edge weights in \cref{eq:edge-weight-online} by \textit{dual-adjusted} edge weights, i.e., for any $\v\in\mathcal V$, round $t$, and edge $e=(\l-1,b_{\l-1},s_{\l-1})\to(\l,b_\l,s_\l)$ in layer $\l\in[\maxbid]$,
\begin{align}\label{eq:edge-weight-online-PD}
\omega_\v^t(e)
:= \frac{1-(v_\l-(\alpha+\lambda_t)b_\l)\,\ind{b_\l\ge \ordotherbid{t}{\l}}}{1+\alpha+1/\rho},
\end{align}
where $\lambda_t\in[0,1/\rho]$ is the dual variable. All edges from layer $\maxbid$ to $d$ have weight $0$. Since $v_\l,b_\l\in[0,1]$ and $\lambda_t\le 1/\rho$, $\omega_\v^t(e)\in[0,1]$. Second, we add an $s$-$d$ edge for the null action $\emptyset$, for which
$\val(\emptyset;\otherbid{t})=\price(\emptyset;\otherbid{t})=0$, and assign it weight $\maxbid/(1+\alpha+1/\rho)$ to preserve the same normalization. Compared with the unconstrained weights, \cref{eq:edge-weight-online-PD} replaces the per-unit surplus $v_\ell-\alpha b_\ell$ by its dual-adjusted counterpart $v_\ell-(\alpha+\lambda_t)b_\ell$. Hence the primal RM maximizes $\ut_{\v^t}(\ibid;\otherbid{t})-\lambda_t\price(\ibid;\otherbid{t})$, penalizing high bids through the current dual price.


\underline{Dual RM.} We use the standard online gradient descent algorithm~\citep{zinkevich2003online} with affine cost functions $g_t\cdot\lambda$ as the dual RM where 
\[
g_t = - (\price(\ibid^t; \otherbid{t})-\rho\maxbid)\in[-\maxbid, \rho\maxbid], ~~\forall t\geq1.
\]
The dual variables are updated following \cref{eq:dual-updates}. Following \cite{castiglioni2022online, castiglioni2023online, fikioris2023approximately}, we can restrict $\lambda_t\in[0, \frac{1}{\rho}]$. Intuitively, overspending relative to $\rho\maxbid$ increases $\lambda_t$, which makes high bids less attractive in subsequent primal updates. Note that the dual RM always operates under full information, independent of the primal RM's feedback model. The procedure is presented in \cref{alg:primal-dual}.


\begin{algorithm}[!tbh]
\caption{Budget Constraint (Full Information)}
\label{alg:primal-dual}
\small{
\begin{algorithmic}[1]
\Require Define $\lambda_1=0$, $B_1=\rho\maxbid T$ and dual learning rates $\zeta_t=\frac{1}{\rho\maxbid\sqrt{t}}, \forall t\geq 1$.
\For{$t = 1, 2, \dots $}
    \State Observe an i.i.d. valuation vector sample $\v^t\sim\mathcal{D}$.
    \If{$B_t\geq \maxbid$}
    \State Submit $\ibid^t\in\bidclass_{\v^t}$ per \cref{alg:weight-pushing-stoc}, \ARef**{line:map-start} to \ref{line:map}.\label{line:PD-primal}
    \Else
    \State Submit $\ibid^t=\emptyset$.
    \EndIf
    \State Observe $\price(\ibid^t; \otherbid{t})$ and set: $B_{t+1} = B_t - \price(\ibid^t; \otherbid{t})$.
    \State Set edge weights $\omega^t_\v(e)$ per \cref{eq:edge-weight-online-PD} for all $\v\in\mathcal{V}$.
    \State Update the dual variable as follows:
    \begin{align}\label{eq:dual-updates}
        \lambda_{t+1} = [\lambda_t + \zeta_t(\price(\ibid^t; \otherbid{t})-\rho\maxbid)]_0^{1/\rho}\,.
    \end{align}
    Here, $[x]_a^b=\min(b, \max(x, a))$ is the Euclidean projection operator on the interval $[a, b]$.
    
\EndFor
\end{algorithmic}}
\end{algorithm}

\begin{theorem}\label{thm:regret-with-budgets} \cref{alg:primal-dual} runs in $O(|\mathcal{V}||\overline{\mathsf{E}}|)=O(|\mathcal{V}|\maxbid^2/\epsilon^3)$ space and time per round and for any $\epsilon\in(0, 1)$ achieves
\[
\rho\cdot\mathsf{Reg}(T)\lesssim \frac{\maxbid}{\rho}+ \frac{\maxbid\sqrt{T}}{\rho} + \mathsf{R}_P^T,
\]
where $\mathsf{R}_P^T$ is
\begin{itemize}
    \item $\frac{\maxbid^{3/2}}{\rho}\sqrt{T\log 1/\epsilon}$ in the full information setting.
    \item $\frac{\maxbid^2}{\rho\epsilon^{3/2}}\sqrt{T\log 1/\epsilon}$ in the bandit setting with known context distribution.
    \item $\frac{\maxbid^2 T^{2/3}(\log 1/\epsilon)^{1/3}}{\rho\epsilon} + \frac{\maxbid^4\log 1/\epsilon}{\rho\epsilon^3}$ in the bandit setting with unknown context distribution.
\end{itemize}
\end{theorem}

\section{Large Number of Contexts}\label{sec:infinite}
In \cref{sec:learning-safe}, we gave learning algorithms whose regret bounds are independent of the number of contexts $|\mathcal V|$, but whose per-round time and space complexity scale as $O(|\mathcal V|)$. This is acceptable when $|\mathcal V|$ is small, but in practice the context space may be large or even infinite. We now show how to remove this dependence while preserving the same regret guarantees. For ease of presentation, we focus on the unconstrained full-information setting; the same shared-coefficient idea extends to $\rho<1$ and to the bandit setting.

Recall the super DAG $\mathcal G(\overline{\mathsf N},\overline{\mathsf E})$ from \cref{sec:learning-safe}, which has the same layered structure as the context-dependent DAGs $\DAG{}$ but is independent of the context. Moreover, $\mathsf E_\v\subseteq \overline{\mathsf E}$ for all $\v\in\mathcal V$.

\textbf{Key Idea.} The efficient implementation relies on three properties: (i) utility is affine in the context, with \textit{coefficients shared across all contexts}; (ii) utility decomposes over edges (\cref{eq:utility-decomp}); and (iii) the weight-pushing algorithm efficiently maps cumulative edge weights to a sampling distribution over paths (\cref{alg:weight-pushing-stoc}). As a result, after observing feedback \emph{ex post}, each round only requires updating a common set of coefficients.
\subsection{Full Information Setting}\label{ssec:infinite-full-info}
We begin by showing that utilities (edge weights) are affine in the context. In the full-information setting, for each $\v\in\mathcal V$ and each edge
$e=(\l-1,b_{\l-1},s_{\l-1})\to(\l,b_\l,s_\l)$ in layer $\l\in[\maxbid]$,
\begin{align}\label{eq:edge-weight-online-alt}
\omega_\v^t(e)
=\frac{1-(v_\l-\alpha b_\l)\cdot\ind{b_\l\ge \ordotherbid{t}{\l}}}{1+\alpha}
=:x^t(e)\,v_\l+y^t(e).
\tag{\eqref{eq:edge-weight-online} restated}
\end{align}
All edges from nodes in layer $\maxbid$ to the destination node $d$ have weight $0$, and
\[
x^t(e)=-\frac{\ind{b_\l\ge \ordotherbid{t}{\l}}}{1+\alpha},
\qquad
y^t(e)=\frac{1+\alpha b_\l\cdot\ind{b_\l\ge \ordotherbid{t}{\l}}}{1+\alpha}.
\]
The coefficients $(x^t(e), y^t(e))$ are context-independent: they depend only on the edge $e$ and the competing bids in round $t$. Therefore, once $\{x^t(e),y^t(e)\}_{e\in\overline{\mathsf E}}$ are known, the edge weights for any context-dependent DAG $\DAG{}$ can be recovered for every $\v\in\mathcal V$.

Once we have defined the coefficients $x^t(e)$ and $y^t(e)$, we now describe how to update them across rounds. To this end, for each edge $e$ define the cumulative (scaled) coefficients
\[
\widehat{x}^{t}(e):=-\eta_{t+1}\sum_{s=1}^{t}x^s(e),
\qquad
\widehat{y}^{t}(e):=-\eta_{t+1}\sum_{s=1}^{t}y^s(e),
\]
for $t\ge 1$, and let $\widehat{x}^{0}(e)=\widehat{y}^{0}(e)=0$. It follows that
\begin{align}
\widehat{x}^{t}(e) &= \frac{\eta_{t+1}}{\eta_t}\,\widehat{x}^{t-1}(e)-\eta_{t+1}x^t(e), \label{eq:main-iteration1}\\
\widehat{y}^{t}(e) &= \frac{\eta_{t+1}}{\eta_t}\,\widehat{y}^{t-1}(e)-\eta_{t+1}y^t(e). \label{eq:main-iteration2}
\end{align}



\textbf{Overview of the Algorithm.} The algorithm is similar to \cref{alg:weight-pushing-stoc} but differs in three key ways: (i) how we instantiate context-dependent DAGs (cf.~\ARef*{line:map-start}); (ii) how we compute the sampling probabilities $\phi_\v^t(\cdot)$ (cf.~\cref{eq:gamma-update,eq:phi-update}); and (iii) how we set the edge weights across contexts (cf.~\ARef*{line:set-wt}). Concretely:

\smallskip
\noindent\textbf{(i) On-the-fly instantiation.} We instantiate the context-dependent DAG only for the current context $\v^t=\v$, rather than for all $\v\in\mathcal V$.

\smallskip
\noindent\textbf{(ii) Computing $\phi_\v^t(\cdot)$.} For any $\v$, set $\Gamma_\v^{t-1}(d)=1$ and compute $\Gamma_\v^{t-1}(\cdot)$ bottom-up. For each node $u=(\l-1,b_{\l-1},s_{\l-1})\in\mathsf N_\v$, define
\begin{align}\label{eq:gamma-update-alt}
\Gamma_\v^{t-1}(u)
=\sum_{v=(\l,b_\l,s_\l):\,u\to v=:e\in\mathsf E_\v}
\Gamma_\v^{t-1}(v)\cdot
\exp\!\big(\widehat{x}^{t-1}(e)\,v_\l+\widehat{y}^{t-1}(e)\big).
\end{align}
Then, for each edge $e=(\l-1,b_{\l-1},s_{\l-1})\to(\l,b_\l,s_\l)\in\mathsf E_\v$, set
\begin{align}\label{eq:phi-update-alt}
\phi_\v^t(e)
=
\exp\!\big(\widehat{x}^{t-1}(e)\,v_\l+\widehat{y}^{t-1}(e)\big)\cdot
\frac{\Gamma_\v^{t-1}(\l,b_\l,s_\l)}{\Gamma_\v^{t-1}(\l-1,b_{\l-1},s_{\l-1})}.
\end{align}
Computing $\phi_\v^t(\cdot)$ for any $\v$ only requires maintaining the shared coefficients $\{\widehat{x}^{t-1}(e), \widehat{y}^{t-1}(e)\}_{e\in\overline{\mathsf E}}$. Once $\phi_\v^t(\cdot)$ is available, we sample a path $\path$ in a Markovian fashion as in \cref{sec:learning-safe}.

\smallskip
\noindent\textbf{(iii) Updating shared coefficients.} Finally, instead of explicitly instantiating per-round edge weights as in \cref{eq:edge-weight-online}, we update the shared coefficients using \cref{eq:main-iteration1,eq:main-iteration2}.

\smallskip
\noindent The full algorithm is given in \cref{alg:weight-pushing-stoc-alt}. Our main result is the following. 

\begin{algorithm}[!tbh]
\caption{\small No Budget Constraints~(Full Information) -- Efficient Implementation}
\label{alg:weight-pushing-stoc-alt}
\small{
\begin{algorithmic}[1]
\Require Learning rates $\eta_t>0, \forall t\geq1$. Define $\eta_0=1$. For all $e\in\overline{\mathsf{E}}$, define $\widehat{x}^0(e)=\widehat{y}^0(e)=0$.
\For{$t = 1, 2, \dots $}
    \State Observe an i.i.d. valuation vector sample $\v^t\sim\mathcal{D}$. Suppose $\v^t=\v$.
    
    \State Construct $\DAG{t}$ and obtain edge probabilities $\phi^t_{\v}(\cdot)$ following \cref{eq:gamma-update-alt} and \cref{eq:phi-update-alt}.
    \State Define initial node $u=s$ and path $\path^t=s$. 
    \While{$u\neq d$}
    \State Sample $v$ with probability $\phi^t_{\v}(u\to v)$.
    \State Append $v$ to the path $\path^t$; set $u\gets v$.
    \EndWhile
    \State Map $\path^t=s\to (1, b_1, s_1)\to\dots\to (\maxbid, b_\maxbid, s_\maxbid)\to d$, and submit $\ibid^t=[b_1, \dots, b_\maxbid]$.\label{line:map-alt}
    \State Update $\widehat{x}^t(e)$ and $\widehat{y}^t(e)$ for all $e\in\overline{\mathsf{E}}$ per \cref{eq:main-iteration1} and \cref{eq:main-iteration2} respectively.
\EndFor
\end{algorithmic}}
\end{algorithm}
\vspace{-0.1cm}
\begin{theorem}\label{thm:regret-safe-alt}
In the full-information setting without budget constraints, \cref{alg:weight-pushing-stoc-alt} implements \cref{alg:weight-pushing-stoc} using $O(|\overline{\mathsf E}|)=O(\maxbid^2/\epsilon^3)$ time and space per round, while achieving the same regret bound. 
\end{theorem}

\subsection{Bandit Setting}\label{apx:efficient-known-bandit}
\textbf{Known Context Distribution.} 
In this setting, for any edge
$e=(\l-1,b_{\l-1},s_{\l-1})\to(\l,b_\l,s_\l)$, the edge-weight estimator $\widehat{\omega}_\v^t(e)$ in \cref{eq:w-hat-estimate-exp3} can be written as
$\widehat{\omega}_\v^t(e)=x^t(e)\cdot v_\l+y^t(e)$, where
\[
x^t(e)=\frac{-\ind{b_\l\ge \ordotherbid{t}{\l}}\cdot\ind{e\in\path^t}}{q^t(e)\,(1+\alpha)},
\qquad
y^t(e)=\frac{\big(1+\alpha b_\l\cdot\ind{b_\l\ge \ordotherbid{t}{\l}}\big)\cdot\ind{e\in\path^t}}{q^t(e)\,(1+\alpha)}.
\]
The main computational bottleneck is evaluating $q^t(e)$, the unconditional probability of selecting edge $e$ in round $t$, which requires taking an expectation with respect to the context distribution $\mathcal D$ (cf.~\cref{eq:def-q}). In \cref{apx:thm:regret-safe-bandit}, we give a procedure to compute $q^t(e)$ in
$O\!\left(|\mathcal V|\cdot \max_{\v\in\mathcal V}|\mathsf E_\v|\right)=O\!\left(|\mathcal V|\,|\overline{\mathsf E}|\right)$ time. Whenever $q^t(e)$ can be computed more efficiently, the same implementation ideas used in the full-information setting apply directly here as well.

\noindent\textbf{Unknown Context Distribution.} Recall that in this setting, depending on the mixing parameter $\delta\in(0,1]$, the algorithm either (i) samples a path from the edge path cover, or (ii) samples according to exponential-weight updates; see the overview in \cref{sssec:bandit} and the full procedure in \cref{alg:weight-pushing-bandit-unk} in \cref{apx:thm:regret-safe-bandit-unk}. In case (i), sampling from the edge path cover requires $O(|\overline{\mathsf E}|)$ time and space per round (see \cref{sssec:bandit}). In case (ii), the shared-coefficient idea from the full-information setting extends to this setting with minor modifications. 

Recall that for edge $e=(\l-1,b_{\l-1},s_{\l-1})\to(\l,b_\l,s_\l)$ in layer $\l\in[\maxbid]$, and valuation vector $\v$,
\begin{align}
    \omega^t_\v(e) = \frac{\alpha+(v_\l-\alpha b_\l)\cdot\ind{b_\l\geq \ordotherbid{t}{\l}}}{1+\alpha}\,.\tag{\eqref{eq:edge-weight-online-unk} restated}
\end{align}
All edges from nodes in layer $\maxbid$ to the destination node $d$ have weight $0$. If $\v^t=\v'$, the edge-weight estimator in this setting is
\begin{align}
    \widehat{\omega}^t_\v(e) =
        \frac{\omega^t_\v(e)}{p^t_{\v'}(e)}\cdot\ind{e\in\path^t},\tag{\eqref{eq:w-hat-estimate-exp3-unk} restated}
\end{align}
where $p^t_{\v'}(e)=\P[e\in\path^t| \v^t=\v']$ is the edge marginal under the mixture distribution in \cref{alg:weight-pushing-bandit-unk}. Following ideas similar to the full information setting, for $\v\in\mathcal{V}$, and edge $e=(\l-1, b_{\l-1}, s_{\l-1})\to(\l, b_{\l}, s_{\l})$ in layer $\l\in[\maxbid]$, we can express
\[
   \widehat{\omega}^t_\v(e) =  x^t(e)\cdot v_\l + y^t(e),
\]
where
\begin{align}\label{eq:xy-bandit-unk-alt}
    x^t(e) = \frac{\ind{b_\l\geq \ordotherbid{t}{\l}}\cdot\ind{e\in\path^t}}{p^t_{\v'}(e)(1+\alpha)}, \qquad y^t(e) = \frac{(\alpha-\alpha b_\l\cdot\ind{b_\l\geq \ordotherbid{t}{\l}})\cdot\ind{e\in\path^t}}{p^t_{\v'}(e)(1+\alpha)}\,.
\end{align}
Then, using the same shared-coefficient idea from the full information setting, we design the learning algorithm in this setting. The full procedure is stated in \cref{alg:weight-pushing-bandit-unk-alt}.

\begin{theorem}\label{thm:regret-safe-bandit-unk-alt} 
In the bandit setting under unknown context distribution without budget constraints, \cref{alg:weight-pushing-bandit-unk-alt} implements \cref{alg:weight-pushing-bandit-unk} using $O(|\overline{\mathsf E}|)=O(\maxbid^2/\epsilon^3)$ time and space per round, while achieving the same regret bound. 
\end{theorem}

\section{Conclusion}
In this work, we studied repeated bidding in multi-unit pay-as-bid auctions under budget constraints. We first considered the unconstrained regime and developed DAG-based algorithms that run in polynomial time and achieve sublinear regret under different feedback models. Building on these ideas, we then introduced a primal–dual framework for the budgeted setting, yielding $\rho$-approximate sublinear regret. Finally, we provided efficient implementations in key settings whose per-round time and space are independent of the number of contexts, extending to large or even infinite context spaces. Several future directions remain open. A key challenge is closing the gap between the regret rates in the bandit setting with known context distribution ($O(\sqrt{T})$) and unknown context distribution ($O(T^{2/3})$). While \citet{schneider2023optimal} obtain $O(\sqrt{T})$ in the unknown-distribution setting when the number of arms is small, achieving comparable rates with combinatorially many actions, similar to our setting, remains elusive. Another direction is to develop efficient algorithms that adapt optimally to both stochastic and adversarial competing bids, which is practically relevant since most bidding environments exhibit mixed behavior. Finally, understanding equilibrium behavior, convergence, and market dynamics when bidders adopt no-regret learning algorithms poses rich theoretical and algorithmic challenges.




\small
\bibliography{references}
\bibliographystyle{plainnat}
\normalsize
\clearpage
\appendix
\crefalias{section}{appendix}
\crefalias{subsection}{appendix}
\section{Omitted Details}\label{apx:model}
\subsection{Suboptimality of the Alternate No-Overbidding Rule}\label{apx:example}
We give an example to show that the alternate notion of per-unit no-overbidding~(NOB), $b_j\leq v_j , \forall j \in[\maxbid]$, which is considered without loss of generality for profit maximizers~($\alpha=1$), can be suboptimal in bidders with cost of capital $\alpha<1$. To see this, fix $\alpha<1$. For any $\epsilon\in(0, \frac{1}{20\maxbid}]$, let 
\begin{align*}
    \v=[1, 1-2\maxbid\epsilon, \dots, 1-2\maxbid\epsilon]\in \R_+^\maxbid\,.
\end{align*}
Consider an auction with competing bid profile $\otherbid{}$ in which each bid is $b=1-2\maxbid\epsilon+\epsilon$. Any $\ibid$ which satisfies $b_\l\leq v_\l, \forall \l\in[\maxbid]$ can obtain at most 1 unit as $b_2\leq v_2=1-2\maxbid\epsilon<b$. To obtain the one possible unit, it is necessary that $b_1\geq b$. Thus, the utility of the bidder is
\begin{align*}
    \ut_\v(\ibid; \otherbid{}) = 1-\alpha b_1\leq 1-\alpha(1-2\maxbid\epsilon + \epsilon).
\end{align*}
Now, define $\ibid'\in\bidclass$ in which each entry is $1-2\maxbid\epsilon+2\epsilon$. For any $j\in[\maxbid]$, 
\begin{align*}
   B_j=\sum_{\l=1}^j b_\l= (1-2\maxbid\epsilon+2\epsilon)j~~\text{and}~~ W_j = \sum_{\l=1}^j v_\l=1+(j-1)(1-2\maxbid\epsilon)
\end{align*}


It can be verified that $\ibid'$ is a NOB strategy per \cref{def:overbid} as $B_j\leq W_j, \forall j \in[\maxbid]$, and obtains $\maxbid$ units since $b<1-2\maxbid\epsilon+2\epsilon$. Thus,
\begin{align*}
   \ut_\v(\ibid'; \otherbid{}) = 1+(\maxbid-1)(1-2\maxbid\epsilon) - \alpha\maxbid(1-2\maxbid\epsilon+2\epsilon)\,.
\end{align*}

Hence, 
\begin{align*}
    \frac{\ut_\v(\ibid'; \otherbid{})}{\ut_\v(\ibid; \otherbid{})} &\geq \frac{1+(\maxbid-1)(1-2\maxbid\epsilon) - \alpha\maxbid(1-2\maxbid\epsilon+2\epsilon)}{1-\alpha(1-2\maxbid\epsilon + \epsilon)}=\frac{(1-\alpha)(\maxbid-2\maxbid(\maxbid-1)\epsilon)}{1-\alpha+\alpha(2\maxbid-1)\epsilon}\\
    &\geq \frac{0.9\maxbid(1-\alpha)}{1-0.1\alpha}\gtrsim \maxbid\,.
\end{align*}   

\subsection{Overbidding May Result in Value $<$ Payment}\label{apx:overbidding-failure}
We show that for any overbidding strategy $\ibid$, there exists a competing-bid profile $\otherbid{}$ such that $\ut_{\v}(\allbids)<0$. Consider a bidder with $\alpha=1$. Recall that if $\ibid$ is an overbidding strategy, then there exists $\ell\in[\maxbid]$ such that $B_\ell>W_\ell$. 

Fix such an $\ell$ and consider a competing-bid profile $\otherbid{}$ in which the top $K-\ell$ competing bids are $2b_1$, and the remaining competing bids are all equal to some $\delta\in\left(0,\frac{b_{\maxbid}}{2}\right]$, where $\delta$ is an integer multiple of $\epsilon$. Under this profile, the bidder wins exactly $\ell$ units. Hence, $\val_{\v}(\allbids)=W_\ell$ and $\price(\allbids)=\sum_{j=1}^{\ell} b_j = B_\ell$, which implies that $\ut_{\v}(\allbids)=W_\ell-B_\ell<0$.

\section{Omitted Proofs}\label{apx:proofs}
\subsection{Proof of \cref{lem:opt-policy}}
For any $\pi\in\safepolicy$, the objective function in \cref{eq:opt-sto} is
\begin{align*}
   \sum_{t=1}^T\E[\ut_{\v^t}(\pi(\v^t); \otherbid{t})]&=\sum_{t=1}^T\sum_{\v\in\mathcal{V}}\P[\v^t=\v]\cdot\ut_{\v}(\pi(\v); \otherbid{t}) =\sum_{\v\in\mathcal{V}}\P[\v^t=\v]\cdot\sum_{t=1}^T\ut_{\v}(\pi(\v); \otherbid{t}),
\end{align*}
where the last equality follows as $\v^t$ is sampled i.i.d. from the distribution $\mathcal{D}$ in each round. Since $\rho\geq1$, the budget constraint is trivially satisfied. As the objective function is separable in $\v$, the optimal stationary policy $\pi^*$ satisfies $\pi^*(\v)\in\argmax_{\pi(\v)\in\bidclass_\v}\sum_{t=1}^T\ut_{\v}(\pi(\v); \otherbid{t})$.

\subsection{Proof of \cref{thm:DAG-uf-policy}}
\textbf{(Path $\leftrightarrow$ strategy).} In $\DAG{}$, consider a $s$-$d$ path, 
\begin{align}\label{eq:example-path}
   \path=s\to (1, b_1, s_1)\to\dots\to (\maxbid, b_\maxbid, s_\maxbid)\to d,
\end{align}
mapped to the strategy $\ibid = [b_1, \dots, b_\maxbid]$. We claim that $\ibid\in\bidclass_\v$. Assume $b_0=\infty$ and $s_0=0$. 

By construction, $b_\l\in\Z_\epsilon$ and $b_{\l-1}\geq b_{\l}, \forall \l\in[\maxbid]$. Furthermore $s_{\l}\stackrel{\eqref{eq:edge-def-cond}}{=}s_{\l-1}+b_\l\implies s_\l=\sum_{j=1}^\l b_j$ as $s_0=0$. Hence, $B_\l=\sum_{j=1}^\l b_j=s_\l\stackrel{\eqref{eq:S-UB}}{\leq} W_\l, \forall \l\in[\maxbid]$ which implies that $\ibid$ is a NOB strategy corresponding to the valuation vector $\v$ per \cref{def:overbid}. Hence, $\ibid\in\bidclass_\v$.

To show bijection, consider any $\ibid=[b_1, \dots, b_\maxbid]\in\bidclass_\v$. With this strategy, associate the $s$-$d$ path stated in \cref{eq:example-path}, where $S_j=\sum_{\l=1}^j b_\l, \forall j\in[\maxbid]$. 

We show that $\path$ is a valid path in the DAG, i.e., each of its nodes and edges exist in $\DAG{}$. As $\ibid\in\bidclass_\v$, $b_j\in\Z_\epsilon$ and $s_j=\sum_{\l=1}^j b_\l \leq W_j, \forall j\in[\maxbid]$. Hence, all nodes in $\path$ exist in the DAG. Furthermore, $b_1\geq \dots\geq b_\maxbid$ and $s_\l = s_{\l-1}+b_\l, \forall \l\in[\maxbid]$ which implies that all the edges exist in $\path$.

\medskip
\noindent\textbf{(Path weight).} Let $\path=s\to (1, b_1, s_1)\to\dots\to (\maxbid, b_\maxbid, s_\maxbid)\to d$ be the path corresponding to $\ibid=[b_1, \dots, b_\maxbid]$. The weight of path $\path$ is
\begin{align*}
    \omega_{\v}(\path)=\sum_{e\in \path}\omega_{\v}(e) &= \sum_{j=1}^\maxbid \omega_{\v}((j-1, b_{j-1}, s_{j-1})\to (j, b_{j}, s_{j}))\\
    &\stackrel{\eqref{eq:edge-weight-offline}}{=}\sum_{j=1}^\maxbid\sum_{t=1}^T\frac{1-(v_j-\alpha b_j)\cdot\ind{b_j\geq \ordotherbid{t}{j}}}{1+\alpha}\\
    &\stackrel{\eqref{eq:utility-decomp}}{=}\frac{\maxbid T-\sum_{t=1}^T\ut_{\v}(\ibid; \otherbid{t})}{1+\alpha}\,.
\end{align*}
Hence, maximizing $\sum_{t=1}^T\ut_{\v}(\ibid; \otherbid{t})$ over $\ibid\in\bidclass_{\v}$ is equivalent to computing the shortest~(minimum-weight) path in the DAG.

\textbf{Space and Time Complexity.} Since $\DAG{}$ is a DAG, the shortest path in the DAG can be computed in $O(|\mathsf{N}_\v|+|\mathsf{E}_\v|)=O(|\mathsf{E}_\v|)=O(\maxbid^2/\epsilon^3)$ space and time complexity since $|\mathsf{N}_\v|\lesssim|\mathsf{E}_\v|$.

\subsection{Proof of \cref{thm:regret-safe}}\label{apx:thm:regret-safe}
Fix $\v\in\mathcal{V}$. Recall that the probability of selecting a path $\path$ in round $t$ given $\v^t=\v$ is 
\begin{align}\label{eq:P-markovian}
    \P_\v^t(\path):=\P[\path^t=\path|\v^t=\v]=\prod_{e\in\path}\phi^t_{\v}(e)\,.
\end{align}
For any edge $e=u\to v\in\mathsf{E}_\v$, the edge probabilities are
\begin{align}
    \phi^t_\v(e)=[\phi^{t-1}_\v(e)]^{\gamma_t}\cdot\exp(-\eta_t \omega^{t-1}_\v(e))\cdot\frac{\Gamma^{t-1}_\v(v)}{\Gamma^{t-1}_\v(u)},\tag{\eqref{eq:phi-update} restated}
\end{align}
where $\Gamma^{t-1}_\v(d)=1$ and for $\Gamma^{t-1}_\v(\cdot)$ is computed recursively in a bottom-up fashion as follows:
    \begin{align}
       \Gamma^{t-1}_\v(u)=\sum_{v:u\to v\in\mathsf{E}_\v}\Gamma^{t-1}_\v(v)\cdot[\phi^{t-1}_\v(u\to v)]^{\gamma_t}\cdot\exp(-\eta_t \omega^{t-1}_\v(u\to v))\,.\tag{\eqref{eq:gamma-update} restated}
    \end{align}
Here, $\eta_0=1$ and
\begin{align}\label{def:gamma-t}
   \gamma_t=\frac{\eta_t}{\eta_{t-1}} ~~\forall t\geq 1\,.
\end{align}
 Now, consider a na\"ive implementation of the Decreasing Hedge algorithm. In this case, the probability of selecting path $\path$ in round $t$ given $\v^t=\v$ is
\begin{align}\label{eq:dec-hedge-naive}
    \HP^{t}_\v(\path) = \frac{\exp(-\eta_t\sum_{s=1}^{t-1}\omega^s_\v(\path))}{\sum_{\path'\in\mathscr{P}_\v}\exp(-\eta_t\sum_{s=1}^{t-1}\omega^s_\v(\path'))},
\end{align}
where $\mathscr{P}_\v$ is the set of all $s$-$d$ paths in $\DAG{}$ and $\omega^s_\v(\path)=\sum_{e\in\path}\omega^s_\v(e)$. Note that if $\eta_t=\eta, \forall t$, then \cref{eq:dec-hedge-naive} describes the action selection probability under the classical Hedge algorithm.
\begin{lemma}\label{lem:dec-hedge-equiv}
    For any $t\in[T]$, $\v\in\mathcal{V}$, and $\path\in\mathscr{P}_\v$, $\P^t_\v(\path)=\HP^t_\v(\path)$. In words, in any round $t$, and $\v\in\mathcal{V}$, the probability of choosing path $\path\in\mathscr{P}_\v$ under \cref{alg:weight-pushing-stoc} is same as that under the na\"ive Decreasing Hedge algorithm.
\end{lemma}

We first establish a few preliminaries necessary to prove \cref{lem:dec-hedge-equiv}. 

\begin{claim}\label{cl:dec-hedge-recursion}
    For any $t\in[T]$, $\v\in\mathcal{V}$, and $\path\in\mathscr{P}_\v$,
    \begin{align*}
    \HP^{t}_\v(\path) &=\frac{[\HP^{t-1}_\v(\path)]^{\gamma_t}\exp(-\eta_t\omega^{t-1}_\v(\path))}{\sum_{\path'\in\mathscr{P}_\v}[\HP^{t-1}_\v(\path')]^{\gamma_t}\exp(-\eta_t\omega^{t-1}_\v(\path'))},
\end{align*}
where $\gamma_t$ is per \cref{def:gamma-t}, where we define $\HP^{0}_\v(\path):=1$ and $\omega^0_\v(\path)=0$ for all $\v\in\mathcal{V}$ and $\path\in\mathscr{P}_\v$.
\end{claim}
\begin{proof}
For any $t\in[T]$, $\v\in\mathcal{V}$, and $\path\in\mathscr{P}_\v$,
   \begin{align*}
    \HP^{t}_\v(\path) &\stackrel{\eqref{eq:dec-hedge-naive}}{=} \frac{\exp(-\eta_t\sum_{s=1}^{t-1}\omega^s_\v(\path))}{\sum_{\path'\in\mathscr{P}_\v}\exp(-\eta_t\sum_{s=1}^{t-1}\omega^s_\v(\path'))}\\
    &=\frac{\exp(-\eta_t\sum_{s=1}^{t-2}\omega^s_\v(\path))\cdot\exp(-\eta_t\omega^{t-1}_\v(\path))}{\sum_{\path'\in\mathscr{P}_\v}\exp(-\eta_t\sum_{s=1}^{t-2}\omega^s_\v(\path'))\cdot\exp(-\eta_t\omega^{t-1}_\v(\path'))}\\
    &=\frac{\exp(-\frac{\eta_t}{\eta_{t-1}}\cdot\eta_{t-1}\sum_{s=1}^{t-2}\omega^s_\v(\path))\cdot\exp(-\eta_t\omega^{t-1}_\v(\path))}{\sum_{\path'\in\mathscr{P}_\v}\exp(-\frac{\eta_t}{\eta_{t-1}}\cdot\eta_{t-1}\sum_{s=1}^{t-2}\omega^s_\v(\path'))\cdot\exp(-\eta_t\omega^{t-1}_\v(\path'))}\\
    &=\frac{[\HP^{t-1}_\v(\path)]^{\gamma_t}\exp(-\eta_t\omega^{t-1}_\v(\path))}{\sum_{\path'\in\mathscr{P}_\v}[\HP^{t-1}_\v(\path')]^{\gamma_t}\exp(-\eta_t\omega^{t-1}_\v(\path'))},
\end{align*} 
where the last line follows from \cref{eq:dec-hedge-naive} corresponding round $t-1$. 
\end{proof}

\begin{claim}\label{cl:gamma-closed-form}
    For any node $u\in\mathsf{N}_\v$, let $\mathscr{P}_\v(u)$ be the set of paths starting at $u$ and terminating in $d$. Then, 
    \begin{align*}
        \Gamma^{t-1}_\v(u)=\sum_{\path\in\mathscr{P}_\v(u)}\prod_{e\in\path} [\phi^{t-1}_\v(e)]^{\gamma_t}\cdot\exp(-\eta_t\omega^{t-1}_\v(e))\,.
    \end{align*}
\end{claim}
\begin{proof}
    We prove the result by backward induction. For the base case, $\Gamma^{t-1}_\v(d)=1$. Suppose the result holds true for all the nodes in layer $\l+1$ for some $0\leq \l\leq \maxbid$. Then, for any node $u$ in layer $\l$, 
    \begin{align*}
        \Gamma^{t-1}_\v(u)&\stackrel{\eqref{eq:gamma-update}}{=}\sum_{v:u\to v\in\mathsf{E}_\v}\Gamma^{t-1}_\v(v)\cdot[\phi^{t-1}_\v(u\to v)]^{\gamma_t}\cdot\exp(-\eta_t \omega^{t-1}_\v(u\to v))\\
        &=\sum_{v:u\to v\in\mathsf{E}_\v}\Big(\sum_{\path\in\mathscr{P}_\v(v)}\prod_{e\in\path} [\phi^{t-1}_\v(e)]^{\gamma_t}\cdot\exp(-\eta_t\omega^{t-1}_\v(e))\Big)\cdot[\phi^{t-1}_\v(u\to v)]^{\gamma_t}\cdot\exp(-\eta_t \omega^{t-1}_\v(u\to v))\\
        &=\sum_{\path\in\mathscr{P}_\v(u)}\prod_{e\in\path} [\phi^{t-1}_\v(e)]^{\gamma_t}\cdot\exp(-\eta_t\omega^{t-1}_\v(e))\,.
    \end{align*}
    Here, the second equality follows from the induction hypothesis.
\end{proof}

Having all the necessary pieces, we now prove \cref{lem:dec-hedge-equiv} by induction on round $t$. 
\begin{proof}[Proof of \cref{lem:dec-hedge-equiv}] For $t=1$, $\HP^1_\v(\path)\stackrel{\eqref{eq:dec-hedge-naive}}{=}\frac{1}{|\mathscr{P}_\v|}$. For $t=1$, 
\begin{align*}
    \P^1_\v(\path)\stackrel{\eqref{eq:P-markovian}}{=}\prod_{e\in\path}\phi^1_\v(e)\stackrel{\eqref{eq:phi-update}}{=}\frac{\Gamma^0_\v(d)}{\Gamma^0_\v(s)},
\end{align*}
where the last equality holds as $\omega^0_\v(\cdot)=0$ and $\phi^0_\v(\cdot)=1$. Recall that $\Gamma^0_\v(d)=1$ and by \cref{cl:gamma-closed-form},
\begin{align*}
   \Gamma^0_\v(s)=\sum_{\path\in\mathscr{P}_\v(s)}1=|\mathscr{P}_\v(s)|=|\mathscr{P}_\v|\implies \P^1_\v(\path)=\frac{1}{|\mathscr{P}_\v|}\,.
\end{align*}
Hence, we have that $\P^1_\v(\path)=\HP^1_\v(\path)$. Suppose the result holds for any round $t-1$, i.e., $\P^{t-1}_\v(\path)=\HP^{t-1}_\v(\path)$. Then, 
\begin{align*}
    \P^t_\v(\path)&\stackrel{\eqref{eq:P-markovian}}{=}\prod_{e\in\path}\phi^t_\v(e)\\
    &\stackrel{\eqref{eq:phi-update}}{=}\prod_{e=u\to v\in\path}[\phi^{t-1}_\v(e)]^{\gamma_t}\cdot\exp(-\eta_t \omega^{t-1}_\v(e))\cdot\frac{\Gamma^{t-1}_\v(v)}{\Gamma^{t-1}_\v(u)}\\
    &=[\HP^{t-1}_\v(\path)]^{\gamma_t}\cdot\exp(-\eta_t\omega^{t-1}_\v(\path))\cdot\frac{\Gamma^{t-1}_\v(d)}{\Gamma^{t-1}_\v(s)},
\end{align*}
where the last line follows as $\HP^{t-1}_\v(\path)=\P^{t-1}_\v(\path)=\prod_{e\in\path}\phi^{t-1}_\v(e)$. Now, recall that $\Gamma^{t-1}_\v(d)=1$ and by \cref{cl:gamma-closed-form},
\begin{align*}
   \Gamma^{t-1}_\v(s) = \sum_{\path'\in\mathscr{P}_\v}\prod_{e\in\path'} [\phi^{t-1}_\v(e)]^{\gamma_t}\cdot\exp(-\eta_t\omega^{t-1}_\v(e))=\sum_{\path'\in\mathscr{P}_\v}[\HP^{t-1}_\v(\path')]^{\gamma_t}\cdot\exp(-\eta_t\omega^{t-1}_\v(\path'))\,.
\end{align*}
Combining everything together and using \cref{cl:dec-hedge-recursion}, we get that
\begin{align*}
    \P^t_\v(\path)=\frac{[\HP^{t-1}_\v(\path)]^{\gamma_t}\cdot\exp(-\eta_t\omega^{t-1}_\v(\path))}{\sum_{\path'\in\mathscr{P}_\v}[\HP^{t-1}_\v(\path')]^{\gamma_t}\cdot\exp(-\eta_t\omega^{t-1}_\v(\path'))}=\HP^t_\v(\path),
\end{align*}
which is the desired result.
\end{proof}

\noindent\textbf{Regret Analysis.} For any $\v\in\mathcal{V}$ and $\eta>0$, define
\begin{align*}
    \Phi^t_\v(\eta):=\frac{1}{\eta}\log\Big(\frac{1}{|\mathscr{P}_\v|}\sum_{\path\in\mathscr{P}_\v}\exp\Big(-\eta \sum_{s=1}^t\omega^s_\v(\path)\Big)\Big)
\end{align*}
Then,
\begin{align*}
    \Phi^t_\v(\eta_t)-\Phi^{t-1}_\v(\eta_t) &=\frac{1}{\eta_t}\log\left(\frac{\sum_{\path\in\mathscr{P}_\v}\exp(-\eta_t \sum_{s=1}^t\omega^s_\v(\path))}{\sum_{\path'\in\mathscr{P}_\v}\exp(-\eta_t \sum_{s=1}^{t-1}\omega^s_\v(\path'))}\right)\\
    &=\frac{1}{\eta_t}\log\left(\frac{\sum_{\path\in\mathscr{P}_\v}\exp(-\eta_t \sum_{s=1}^{t-1}\omega^s_\v(\path))\cdot\exp(-\eta_t\omega^t_\v(\path))}{\sum_{\path'\in\mathscr{P}_\v}\exp(-\eta_t \sum_{s=1}^{t-1}\omega^s_\v(\path'))}\right)\\
    &\stackrel{\eqref{eq:dec-hedge-naive}}{=}\frac{1}{\eta_t}\log\left(\sum_{\path\in\mathscr{P}_\v}\HP^t_\v(\path)\cdot\exp(-\eta_t\omega^t_\v(\path))\right)\\
    &=\frac{1}{\eta_t}\log\left(\sum_{\path\in\mathscr{P}_\v}\P^t_\v(\path)\cdot\exp(-\eta_t\omega^t_\v(\path))\right)\,.
\end{align*}
By \cref{eq:P-markovian}, $\P^t_\v(\path)=\P[\path^t=\path|\v^t=\v]$. For $x\geq0$, $e^{-x}\leq 1-x+x^2$. Thus,
\begin{align*}
 \Phi^t_\v(\eta_t)-\Phi^{t-1}_\v(\eta_t) &\leq  \frac{1}{\eta_t}\log\left(\sum_{\path\in\mathscr{P}_\v}\P[\path^t=\path|\v^t=\v](1-\eta_t\omega^t_\v(\path)+\eta_t^2\omega^t_\v(\path)^2)\right)\\
 &=\frac{1}{\eta_t}\log\left(1+\sum_{\path\in\mathscr{P}_\v}\P[\path^t=\path|\v^t=\v](-\eta_t\omega^t_\v(\path)+\eta_t^2\omega^t_\v(\path)^2)\right)\,.
\end{align*}
Since $-\eta_t\omega^t_\v(\path)+\eta_t^2\omega^t_\v(\path)^2\geq -\tfrac{1}{4}$, $\sum_{\path\in\mathscr{P}_\v}\P[\path^t=\path|\v^t=\v](-\eta_t\omega^t_\v(\path)+\eta_t^2\omega^t_\v(\path)^2)\geq -\tfrac{1}{4}$. Moreover, $\log(1+x)\leq x$ for all $x>-1$. Thus,
\begin{align*}
   \Phi^t_\v(\eta_t)-\Phi^{t-1}_\v(\eta_t)  
   &\leq -\sum_{\path\in\mathscr{P}_\v}\P[\path^t=\path|\v^t=\v]\omega^t_\v(\path) + \eta_t\sum_{\path\in\mathscr{P}_\v}\P[\path^t=\path|\v^t=\v]\omega^t_\v(\path)^2
\end{align*}
Summing from $t=1$ to $t=T$ and using $\Phi^{0}_\v(\eta_1)=0$, the left hand side becomes 
\begin{align*}
   \sum_{t=1}^T\left(\Phi^t_\v(\eta_t)-\Phi^{t-1}_\v(\eta_t) \right) = \Phi^T_\v(\eta_T)-\Phi^{0}_\v(\eta_1) +\sum_{t=1}^{T-1}\left(\Phi^t_\v(\eta_{t})-\Phi^{t}_\v(\eta_{t+1})\right)\,.
\end{align*}

By \cref{thm:DAG-uf-policy}, there is a bijection between paths $\path\in\mathscr P_\v$ and strategies $\ibid\in\bidclass_\v$. Recall that
\[
\safepolicy=\Big\{\pi:\mathcal V\to \bigcup_{\v\in\mathcal V}\bidclass_\v \ \text{s.t.}\ \pi(\v)\in \bidclass_\v\Big\}.
\]
Composing this bijection with the policy mapping, we define a path-based policy class as follows:
\begin{align}\label{eq:tilde-policy-class}
  \widetilde{\safepolicy}:=\Big\{\widetilde{\pi}: \mathcal{V}\to\bigcup_{\v\in\mathcal{V}}\mathscr{P}_\v, ~\text{s.t.}~ \widetilde{\pi}(\v)\in\mathscr{P}_\v\Big\}\,.  
\end{align}
By construction, there is a one-to-one correspondence between $\pi\in\safepolicy$ and $\widetilde{\pi}\in\widetilde{\safepolicy}$. So, for any $\widetilde{\pi}\in\widetilde{\safepolicy}$, 
\begin{align*}
    \Phi^T_\v(\eta_T)\geq \frac{1}{\eta_T}\log\left(\frac{1}{|\mathscr{P}_\v|}\cdot\exp\left(-\eta_T\sum_{t=1}^T\omega^t_\v(\widetilde{\pi}(\v))\right)\right) = -\frac{\log |\mathscr{P}_\v|}{\eta_T} - \sum_{t=1}^T\omega^t_\v(\widetilde{\pi}(\v)) \\
    \implies \sum_{t=1}^T\left(\Phi^t_\v(\eta_t)-\Phi^{t-1}_\v(\eta_t) \right) \geq  -\frac{\log |\mathscr{P}_\v|}{\eta_T} - \sum_{t=1}^T\omega^t_\v(\widetilde{\pi}(\v))+ \sum_{t=1}^{T-1}\left(\Phi^t_\v(\eta_{t})-\Phi^{t}_\v(\eta_{t+1})\right)\,.
\end{align*}
This implies
\begin{align}\label{eq:pre-final-regret-bound}
   -&\frac{\log |\mathscr{P}_\v|}{\eta_T} - \sum_{t=1}^T\omega^t_\v(\widetilde{\pi}(\v))+ \sum_{t=1}^{T-1}\left(\Phi^t_\v(\eta_{t})-\Phi^{t}_\v(\eta_{t+1})\right)\nonumber \\
   &\leq  -\sum_{t=1}^T\sum_{\path\in\mathscr{P}_\v}\P[\path^t=\path|\v^t=\v]\omega^t_\v(\path) + \sum_{t=1}^T\sum_{\path\in\mathscr{P}_\v}\P[\path^t=\path|\v^t=\v]\eta_t\omega^t_\v(\path)^2\,. 
\end{align}

\begin{claim}\label{cl:non-decreasing-func}
    For any $\v\in\mathcal{V}$ and $t\in[T]$, $\Phi^t_\v(\eta)$ is non-decreasing for $\eta>0$.
\end{claim}
\begin{proof}
Recall that $\Phi^t_\v(\eta)=\frac{1}{\eta}\log\left(\frac{1}{|\mathscr{P}_\v|}\sum_{\path\in\mathscr{P}_\v}\exp\left(-\eta \sum_{s=1}^t\omega^s_\v(\path)\right)\right)$. Define
    \begin{align*}
        g(\eta)&=\log\left(\frac{1}{|\mathscr{P}_\v|}\sum_{\path\in\mathscr{P}_\v}\exp\left(-\eta \sum_{s=1}^t\omega^s_\v(\path)\right)\right),
    \end{align*}
    such that $\Phi^t_\v(\eta)=g(\eta)/\eta$. Note that $g(\eta)$ is convex because it is the composition of a convex function~(LogSumExp) with an affine map of $\eta$,$-\eta \sum_{s=1}^t\omega^s_\v(\path)$, plus a constant, $-\log |\mathscr{P}_\v|$. As $\Phi^t_\v(\eta)$ is differentiable, we have
\begin{align*}
    [\Phi^t_\v(\eta)]' = \frac{\eta g'(\eta)-g(\eta)}{\eta^2}.
\end{align*}
    Define $h(\eta)=\eta g'(\eta)-g(\eta)$ which implies $h'(\eta)=\eta g''(\eta)$. As $g(\eta)$ is convex and twice differentiable, $g''(\eta)\geq0$ which implies for $\eta>0$, $h'(\eta)\geq0$, i.e., $h(\eta)$ is a non-decreasing function. Furthermore, $h(0)=-g(0)=0$. Hence, $h(\eta)\geq0$, for all $\eta>0$ which implies $[\Phi^t_\v(\eta)]' = h(\eta)/\eta\geq 0, \forall \eta>0$. Thus, $\Phi^t_\v(\eta)$ is non-decreasing in $\eta>0$.
\end{proof}

    Since $\eta_t\geq \eta_{t+1}$, by \cref{cl:non-decreasing-func}, we have $\sum_{t=1}^{T-1}\left(\Phi^t_\v(\eta_{t})-\Phi^{t}_\v(\eta_{t+1})\right)\geq 0$. Thus, after rearranging and using the fact that $|\mathscr{P}_\v|\lesssim \Big(\frac{1}{\epsilon}\Big)^{\maxbid}$, \cref{eq:pre-final-regret-bound} becomes 
    \begin{align}\label{eq:regret-stoc-in-path}
        \sum_{t=1}^T\sum_{\path\in\mathscr{P}_\v}\P[\path^t=\path|\v^t=\v]\omega^t_\v(\path) - \sum_{t=1}^T\omega^t_\v(\widetilde{\pi}(\v))&\lesssim \frac{\maxbid\log 1/\epsilon}{\eta_T} + \sum_{t=1}^T\sum_{\path\in\mathscr{P}_\v}\P[\path^t=\path|\v^t=\v]\eta_t\omega^t_\v(\path)^2\,. 
    \end{align}
Recall that for edge $e=(\l-1, b_{\l-1}, s_{\l-1})\to(\l, b_{\l}, s_{\l})$ in layer $\l\in[\maxbid]$, the edge weight is:
\begin{align}
    \omega^t_\v(e) &= \frac{1-(v_\l-\alpha b_\l)\cdot\ind{b_\l\geq \ordotherbid{t}{\l}}}{1+\alpha}\,.\tag{\eqref{eq:edge-weight-online} restated}\\
    \implies\omega^t_\v(\path) &=\sum_{e\in\path}\omega^t_\v(e)\stackrel{\eqref{eq:utility-decomp}}{=}\frac{M-\ut_{\v}(\ibid; \otherbid{t})}{1+\alpha},\label{eq:path-weight-to-utility}
\end{align} 
where $\ibid\in\bidclass_\v$ is bidding strategy corresponding to the path $\path\in\mathscr{P}_\v$. Using the facts that \(\omega^t_\v(\path)\leq \maxbid\) and existence of a bijective mapping between policies in $\widetilde{\safepolicy}$ and $\safepolicy$, \cref{eq:regret-stoc-in-path} becomes
\begin{align*}
   &\sum_{t=1}^T\ut_{\v}(\pi(\v); \otherbid{t}) - \sum_{t=1}^T\sum_{\ibid\in\bidclass_\v}\P[\ibid^t=\ibid|\v^t=\v]\ut_{\v}(\ibid; \otherbid{t})\\
   &\lesssim (1+\alpha)\left[\frac{\maxbid\log 1/\epsilon}{\eta_T} + \maxbid^2\sum_{t=1}^T\eta_t\sum_{\ibid\in\bidclass_\v}\P[\ibid^t=\ibid|\v^t=\v]\right]\\
    &=(1+\alpha)\left[\frac{\maxbid\log 1/\epsilon}{\eta_T} + \maxbid^2\sum_{t=1}^T\eta_t\right] 
\end{align*}

Taking expectations with respect to the randomness of the valuation vectors, and using the facts that $\alpha\leq 1$, and that the choice of ${\pi}$ was arbitrary, we obtain
\begin{align*}
   \mathsf{Reg}_{nb}(T) &=\mathsf{OPT}_{nb}-\sum_{t=1}^T\E[\ut_{\v^t}(\allbids^t)]\\
   &\stackrel{\eqref{eq:opt-sto-NB}}{=}\max_{\pi\in\safepolicy}\sum_{t=1}^T\E[\ut_{\v^t}(\pi(\v^t); \otherbid{t})] - \sum_{t=1}^T\E[\ut_{\v^t}(\allbids^t)] \lesssim \frac{\maxbid\log 1/\epsilon}{\eta_T} + \maxbid^2\sum_{t=1}^T\eta_t,
\end{align*}
where $\allbids^t=(\ibid; \otherbid{t}), \forall t\geq 1$. Setting $\eta_t= \sqrt{\frac{\log1/\epsilon}{\maxbid t}}$, and using the fact that $\sum_{t=1}^T\frac{1}{\sqrt{t}}\leq 2\sqrt{T}$, we get
\begin{align*}
    \mathsf{Reg}_{nb}(T)\lesssim \maxbid^{3/2}\sqrt{T\log 1/\epsilon}\,.
\end{align*}

\textbf{Space and Time Complexity.} The main bottleneck of implementing \cref{alg:weight-pushing-stoc} in the full information setting is updating the edge probabilities, $\phi^t_\v(\cdot)$ for all the contexts. For a fixed $\v\in\mathcal{V}$, updating $\phi^t_\v(\cdot)$ requires a forward pass and a backward pass over the edges $\mathsf{E}_\v$. Thus, \cref{alg:weight-pushing-stoc} requires $O(|\mathcal{V}|\cdot\max_{\v\in\mathcal{V}}|\mathsf{E}_\v|)=O(|\mathcal{V}|\maxbid^2/\epsilon^3)$ space and time per round.

\subsection{Proof of \cref{thm:regret-safe-bandit}}\label{apx:thm:regret-safe-bandit}
Recall that for any $\v\in\mathcal{V}$, $e\in\mathsf{E}_\v$ and $t\in[T]$, 
\begin{align}
    \widehat{\omega}^t_\v(e) &=
        \frac{\omega^t_\v(e)}{q^t(e)}\cdot\ind{e\in\path^t},\tag{\eqref{eq:w-hat-estimate-exp3} restated}
\end{align}
where $q^t(e)=\sum_{\v\in\mathcal{V}}\P[\v^t=\v]\sum_{\path\in\mathscr{P}_\v:e\in\path}\P[\path^t=\path|\v^t=\v]$. Then,
\begin{lemma}\label{lem:properties-bandit-known}
    For any $\v\in\mathcal{V}$, $e\in\mathsf{E}_\v$ and $t\in[T]$, 
    \begin{enumerate}
        \item $\widehat{\omega}^t_\v(e)$ is well defined,
        \item $\widehat{\omega}^t_\v(e)\in[0, \infty)$,
        \item $\E[\widehat{\omega}^t_\v(e)]={\omega}^t_\v(e)$, and
        \item $\E[\widehat{\omega}^t_\v(e)^2]\leq \frac{1}{q^t(e)}$\,.
    \end{enumerate}
    Here, the expectation is taken with respect to the randomness in contexts as well as the learning algorithm conditioned on the history up to round $t$.
\end{lemma}
\begin{proof}
Fix any $\v\in\mathcal{V}$, $e\in\mathsf{E}_\v$ and $t\in[T]$.
\begin{enumerate}
    \item Without loss of generality, assume $\P[\v^t=\v]>0$.\footnote{If there exists context $\v\in\mathcal{V}$ such that $\P[\v^t=\v]=0,$ consider $\mathcal{V}'=\mathcal{V}\setminus\{\v\}$.} This implies that $q^t(e)>0$, since every $e\in\mathsf{E}_\v$ lies on some path $\path\in\mathscr{P}_\v$ and every $\path$ satisfies $\P[\path^1=\path| \v^1=\v]>0$. The edge probabilities updates in \cref{alg:weight-pushing-stoc} ensure that $\P[\path^t=\path|\v^t=\v]>0$ for all subsequent rounds $t$. Hence, $\widehat{\omega}^t_\v(e)$ is well defined. 
    \item The result follows directly from the definition in \cref{eq:w-hat-estimate-exp3}.
    \item Since
    \begin{align*}
\E\left[\ind{e\in\path^t}\right]=\sum_{\v\in\mathcal{V}}\P[\v^t=\v]\sum_{\path\in\mathscr{P}_\v:e\in\path}\P[\path^t=\path|\v^t=\v]=q^t(e), 
    \end{align*}
    we get that $\E[\widehat{\omega}^t_\v(e)]={\omega}^t_\v(e)$.
    \item By definition, 
    \begin{align}\label{eq:omega-hat-UB}
  \E[\widehat{\omega}^t_\v(e)^2] &\stackrel{\eqref{eq:w-hat-estimate-exp3}}{=} \frac{\omega^t_\v(e)^2}{q^t(e)^2}\cdot\E\left[\ind{e\in\path^t}\right] \leq \frac{1}{q^t(e)},
\end{align}
where the first inequality follows as $0\leq \omega^t_\v(e)\leq 1$~(see \cref{eq:edge-weight-online}).
\end{enumerate}
\end{proof}

\noindent Following the proof of \cref{thm:regret-safe} up to \cref{eq:regret-stoc-in-path},\footnote{The analysis of \cref{thm:regret-safe} up to \cref{eq:regret-stoc-in-path} only requires $\widehat{\omega}^t_\v(\cdot)\ge 0$ which is ensured by \cref{lem:properties-bandit-known} (2).} for any stationary policy $\widetilde{\pi} \in \widetilde{\safepolicy}$, we have
\begin{align}\label{eq:bandit-UB-intermediate}
&\sum_{t=1}^T\sum_{\path\in\mathscr{P}_\v}\P[\path^t=\path|\v^t=\v]\widehat{\omega}^t_\v(\path) - \sum_{t=1}^T\widehat{\omega}^t_\v(\widetilde{\pi}(\v))\nonumber\\
&\lesssim \frac{\maxbid\log 1/\epsilon}{\eta_T} + \sum_{t=1}^T\sum_{\path\in\mathscr{P}_\v}\P[\path^t=\path|\v^t=\v]\eta_t\widehat{\omega}^t_\v(\path)^2\nonumber\\
&\leq \frac{\maxbid\log 1/\epsilon}{\eta_T} + \maxbid \sum_{t=1}^T\eta_t\sum_{\path\in\mathscr{P}_\v}\P[\path^t=\path|\v^t=\v]\sum_{e\in\path}\widehat{\omega}^t_\v(e)^2,
\end{align}
where the last line follows due to Cauchy-Schwarz inequality.
 
Taking expectations with respect to the randomness of contexts and \cref{alg:weight-pushing-stoc} conditioned on the history and using \cref{eq:omega-hat-UB}, we have 
\begin{align*}
  \sum_{t=1}^T\sum_{\path\in\mathscr{P}_\v}\P[\path^t=\path|\v^t=\v]{\omega}^t_\v(\path) - \sum_{t=1}^T{\omega}^t_\v(\widetilde{\pi}(\v))&\lesssim \frac{\maxbid\log 1/\epsilon}{\eta_T} + \maxbid \sum_{t=1}^T\eta_t\sum_{\path\in\mathscr{P}_\v}\P[\path^t=\path|\v^t=\v]\sum_{e\in\path}\frac{1}{q^t(e)}\,. 
\end{align*}
Relating the path weights in the DAG to the utility using \cref{eq:path-weight-to-utility} and leveraging the bijective mapping between policies in $\widetilde{\safepolicy}$ and $\safepolicy$, the left hand side becomes
\[
\sum_{t=1}^T\ut_{\v}(\pi(\v); \otherbid{t})- \sum_{t=1}^T\sum_{\ibid\in\bidclass_\v}\P[\ibid^t=\ibid|\v^t=\v]\ut_{\v}(\ibid; \otherbid{t})\,.
\]

Taking expectations with respect to the contexts and the history, we get
\begin{align*}
   \mathsf{Reg}_{nb}(T) &=\sum_{t=1}^T\E[\ut_{\v^t}(\pi(\v^t); \otherbid{t})]- \sum_{t=1}^T\E[\ut_{\v^t}(\allbids^t)]\\
   &\lesssim\frac{\maxbid \log 1/\epsilon}{\eta_T} + \maxbid\E\left[\sum_{t=1}^T\eta_t\sum_{\v\in\mathcal{V}}\P[\v^t=\v]\sum_{e\in \mathsf{E}_\v}\frac{1}{q^t(e)}\sum_{\path\in\mathscr{P}_\v:e\in\path}\P[\path^t=\path|\v^t=\v]\right]
\end{align*}

Define $\mathsf{E}:=\bigcup_{\v\in\mathcal{V}}\mathsf{E}_\v$. This implies
\begin{align*}
\mathsf{Reg}_{nb}(T)
&\lesssim \frac{\maxbid \log 1/\epsilon}{\eta_T} + \maxbid\E\left[\sum_{t=1}^T\eta_t\sum_{e\in\mathsf{E}}\frac{1}{q^t(e)}\sum_{\v\in\mathcal{V}}\P[\v^t=\v]\sum_{\path\in\mathscr{P}_\v:e\in\path}\P[\path^t=\path|\v^t=\v]\right]\\
&\stackrel{\eqref{eq:def-q}}{=}\frac{\maxbid \log 1/\epsilon}{\eta_T} + \maxbid\E\left[\sum_{t=1}^T\eta_t\sum_{e\in \mathsf{E}}\frac{1}{q^t(e)}\cdot q^t(e)\right]\\
   &= \frac{\maxbid \log 1/\epsilon}{\eta_T} + \maxbid |\mathsf{E}|\sum_{t=1}^T\eta_t\,.
\end{align*}

To bound $|\mathsf E|$, recall the super DAG $\mathcal G(\overline{\mathsf N},\overline{\mathsf E})$ constructed in \cref{sec:infinite} and that $\mathsf E_\v\subseteq \overline{\mathsf E}$ for all $\v\in\mathcal V$. Hence, $\mathsf E= \bigcup_{\v\in\mathcal V}\mathsf E_\v \subseteq \overline{\mathsf E}$. Moreover, $|\overline{\mathsf E}|=O\!\left(\frac{\maxbid^2}{\epsilon^3}\right)$. Since, $|\mathsf E|\le |\overline{\mathsf E}|$, we have

\begin{align*}
    \mathsf{Reg}_{nb}(T)\lesssim\frac{\maxbid \log 1/\epsilon}{\eta_T} + \frac{\maxbid^3}{\epsilon^3}\sum_{t=1}^T\eta_t\,.
\end{align*}

Setting $\eta_t=\frac{1}{\maxbid}\sqrt{\frac{\epsilon^3\log 1/\epsilon}{t}}$, and using the fact that $\sum_{t=1}^T\frac{1}{\sqrt{t}}\leq 2\sqrt{T}-1$, we get
\begin{align*}
    \mathsf{Reg}_{nb}(T)\lesssim \frac{\maxbid^2}{\epsilon^{3/2}}\sqrt{T\log 1/\epsilon}\,.
\end{align*}
\textbf{Space and Time Complexity.} In this setting, the main bottleneck of implementing \cref{alg:weight-pushing-stoc} is computing $q^t(e)$~(equivalently $\sum_{\path\in\mathscr{P}_\v:e\in\path}\P[\path^t=\path|\v^t=\v]$ for each $\v\in\mathcal{V}$) per \cref{eq:def-q}. For a fixed $\v\in\mathcal{V}$ and $e\in\mathsf{E}_\v$, 
\begin{align}\label{eq:sum-prod}
   \sum_{\path\in\mathscr{P}_\v:e\in\path}\P[\path^t=\path|\v^t=\v] \stackrel{\eqref{eq:P-markovian}}{=}  \sum_{\path\in\mathscr{P}_\v:e\in\path}\prod_{e'\in\path}\phi^t_\v(e')
\end{align}
Leveraging \citet[Theorem 4.4]{golrezaei2025learning}, for any $\v\in\mathcal{V}$, \cref{eq:sum-prod} can be computed in $O(|\mathsf{E}_\v|)$ time and space per round. Hence, \cref{alg:weight-pushing-stoc} runs in $O(|\mathcal{V}|\cdot\max_{\v\in\mathcal{V}}|\mathsf{E}_\v|)=O\Big(\frac{ |\mathcal{V}|\maxbid^2}{\epsilon^3}\Big)$ time and space per round.

\subsection{Proof of \cref{thm:regret-safe-bandit-unk}}\label{apx:thm:regret-safe-bandit-unk}
We first present the learning algorithm in the bandit setting with unknown context distribution.

\begin{algorithm}[!tbh]
\caption{\small No Budget Constraints~(Bandit Feedback, Unknown Distribution)}
\label{alg:weight-pushing-bandit-unk}
\small{
\begin{algorithmic}[1]
\Require Set of valuation vectors $\mathcal V$, learning rates $\eta_t>0$, exploration parameter $\delta\in(0,1]$, and edge path cover $\mathcal C$ of the super DAG.
Define $\eta_0=1$, $\phi^0_\v(e)=1$, and $\widehat\omega^0_\v(e)=0$, $\forall e\in{\overline{\mathsf E}}$.
\For{$t=1,2,\dots$}
    \State Observe an i.i.d. valuation vector sample $\v^t\sim\mathcal D$.
    \State Construct $\DAG{t}$ without weights $\forall\v\in\mathcal{V}$.\label{line:map-unk-start}
    \State Sample $Z_t\sim\mathrm{Unif}[0,1]$.
    \For{$\v\in\mathcal V$}
        \State Obtain edge probabilities $\phi^t_\v(\cdot)$ following \cref{eq:gamma-update} and \cref{eq:phi-update}, with $\omega^{t-1}_\v(\cdot)$ replaced by $-\widehat\omega^{t-1}_\v(\cdot)$, where $\widehat\omega^{t-1}_\v(\cdot)$ is defined in \cref{eq:w-hat-estimate-exp3-unk}.
    \EndFor
    \If{$Z_t\le\delta$}
        \State Select path $\path^t$ uniformly at random from $\mathcal C$.
    \Else
        \State Define initial node $u=s$ and path $\path^t=s$.
        \While{$u\neq d$}
            \State Sample $v$ with probability $\phi^t_{\v^t}(u\to v)$.
            \State Append $v$ to the path $\path^t$; set $u\gets v$.
        \EndWhile
    \EndIf
    \State Map $\path^t=s\to(1,b_1,s_1)\to\dots\to(\maxbid,b_\maxbid,s_\maxbid)\to d$, and submit $\ibid^t=[b_1,\dots,b_\maxbid]$.\label{line:map-unk-end}
    \State Set estimated edge gains per \cref{eq:w-hat-estimate-exp3-unk} for all $\v\in\mathcal V$ and $e\in{\overline{\mathsf E}}$.
\EndFor
\end{algorithmic}}
\end{algorithm}

\noindent Recall that for any $\v\in\mathcal{V}$, $e\in\overline{\mathsf{E}}$ and $t\in[T]$, if $\v^t=\v'$, the estimator is given as 
\begin{align}
    \widehat{\omega}^t_\v(e) &=
        \frac{\omega^t_\v(e)}{p^t_{\v'}(e)}\cdot\ind{e\in\path^t},\tag{\eqref{eq:w-hat-estimate-exp3-unk} restated}
\end{align}
where $\omega^t_\v(e)$ is per \cref{eq:edge-weight-online-unk}, and $p^t_{\v'}(e)=\P[e\in\path^t| \v^t=\v']$ is the edge marginal under the mixture distribution in \cref{alg:weight-pushing-bandit-unk}.

\begin{lemma}\label{lem:properties-bandit-unknown}
For any $\v\in\mathcal V$, $e\in{\overline{\mathsf{E}}}$, and $t\in[T]$,
\begin{enumerate}
    \item $\widehat{\omega}^t_\v(e)$ is well-defined;
    \item $\widehat{\omega}^t_\v(e)\in[0,|\overline{\mathsf E}|/\delta]$;
    \item $\E[\widehat{\omega}^t_\v(e)]=\omega^t_\v(e)$;
    \item $\E[(\widehat{\omega}^t_\v(e))^2]\le |\overline{\mathsf E}|/\delta$.
\end{enumerate}
Here, $\overline{\mathsf E}$ is the edge set of the super DAG
$\overline{\mathcal G}=(\overline{\mathsf N},\overline{\mathsf E})$ constructed in \cref{sec:learning-safe}, and $\E[\cdot]$ denotes expectation over the randomness of the context and the learning algorithm in round $t$, conditional on the history before round $t$.
\end{lemma}



\begin{proof}
Following the proof of \cref{thm:regret-safe}, Lines 6 and 10--13 of
\cref{alg:weight-pushing-bandit-unk} implement Decreasing Hedge over
$s$-$d$ paths in the context-dependent DAG $\xDAG{t}{\v'}$. In particular, conditional on
$\v^t=\v'$, the exponential-weights component selects path $\path$ with
probability
\begin{align}\label{eq:hedge-unk}
    \widetilde{\P}[\path^t=\path\mid \v^t=\v']
    =
    \frac{
        \exp\!\left(\eta_t\sum_{s=1}^{t-1}\widehat{\omega}^{s}_{\v'}(\path)\right)
    }{
        \sum_{\path'\in{\mathscr P_{\v'}}}
        \exp\!\left(\eta_t\sum_{s=1}^{t-1}\widehat{\omega}^{s}_{\v'}(\path')\right)
    } .
\end{align}
We extend $\widetilde{\P}[\cdot\mid \v^t=\v']$ to $\overline{\mathscr P}$ by setting it equal to zero outside $\mathscr P_{\v'}$, where $\overline{\mathscr P}$ is the set of $s$-$d$ paths in the super DAG. Therefore, under the mixture distribution used by
\cref{alg:weight-pushing-bandit-unk},
\begin{align}\label{eq:def-P}
    \P[\path^t=\path\mid \v^t=\v']
    =
    (1-\delta)\widetilde{\P}[\path^t=\path\mid \v^t=\v']
    +
    \delta\cdot
    \frac{\textsc{Count}[\path\in\mathcal C]}{|\mathcal C|},
\end{align}
where $\textsc{Count}[\path\in\mathcal C]$ is the multiplicity of
$\path$ in the edge path cover $\mathcal C$.

\begin{enumerate}
\item
Fix any $\v'\in\mathcal V$. For every $e\in{\overline{\mathsf{E}}}$,
\begin{align}\label{eq:min-prob-LB}
    p^t_{\v'}(e)
    &=
    \sum_{\path\in\overline{\mathscr P}:e\in\path}
    \P[\path^t=\path\mid \v^t=\v'] \ge
    \delta
    \sum_{\path\in\overline{\mathscr P}:e\in\path}
    \frac{\textsc{Count}[\path\in\mathcal C]}{|\mathcal C|}
    \ge
    \frac{\delta}{|\mathcal C|}
    \ge
    \frac{\delta}{|\overline{\mathsf E}|}.
\end{align}
The second inequality uses that $\mathcal C$ covers every edge, and the
last inequality uses $|\mathcal C|\le |\overline{\mathsf E}|$. Hence
$p^t_{\v'}(e)>0$, so \cref{eq:w-hat-estimate-exp3-unk} is well-defined.

\item
Since $\omega^t_\v(e)\in[0,1]$ by
\cref{eq:edge-weight-online-unk}, we have
$\widehat{\omega}^t_\v(e)\ge0$. Moreover, by
\cref{eq:min-prob-LB},
\[
    \widehat{\omega}^t_\v(e)
    =
    \frac{\omega^t_\v(e)}{p^t_{\v'}(e)}
    \ind{e\in\path^t}
    \le
    \frac{1}{p^t_{\v'}(e)}
    \le
    \frac{|\overline{\mathsf E}|}{\delta}.
\]

\item
Condition on $\v^t=\v'$ and on the pre-sampling history in round $t$.
Then
\begin{align*}
    \E\!\left[\widehat{\omega}^t_\v(e)\mid \v^t=\v'\right]
    &=
    \frac{\omega^t_\v(e)}{p^t_{\v'}(e)}
    \E\!\left[\ind{e\in\path^t}\mid \v^t=\v'\right] =
    \frac{\omega^t_\v(e)}{p^t_{\v'}(e)}
    p^t_{\v'}(e)
    =
    \omega^t_\v(e).
\end{align*}
Taking expectation over $\v^t$ gives
$\E[\widehat{\omega}^t_\v(e)]=\omega^t_\v(e)$.

\item
Again conditioning on $\v^t=\v'$,
\begin{align*}
    \E\!\left[(\widehat{\omega}^t_\v(e))^2\mid \v^t=\v'\right]
    &=
    \frac{(\omega^t_\v(e))^2}{(p^t_{\v'}(e))^2}
    \E\!\left[\ind{e\in\path^t}\mid \v^t=\v'\right]=
    \frac{(\omega^t_\v(e))^2}{p^t_{\v'}(e)}
    \le
    \frac{1}{p^t_{\v'}(e)}
    \le
    \frac{|\overline{\mathsf E}|}{\delta},
\end{align*}
where the first inequality uses $\omega^t_\v(e)\in[0,1]$ and the last
uses \cref{eq:min-prob-LB}. Taking expectation over $\v^t$ gives
$\E[(\widehat{\omega}^t_\v(e))^2]\le |\overline{\mathsf E}|/\delta$.
\end{enumerate}
\end{proof}

We now prove the bound on the number of overbidding rounds. In any round $t$, the learner samples from the path set $\mathscr P_{\v^t}$ with probability $1-\delta$. On this event, the submitted bid lies in $\mathcal B_{\v^t}$. Therefore, overbidding can occur only when the learner samples from the edge path cover $\mathcal C$, which happens with probability $\delta$. Let
\[
    Z_t:=\ind{\text{the learner samples from } \mathcal C \text{ in round }t}.
\]
Then
\[
    \ind{\ibid^t\notin\mathcal B_{\v^t}}\le Z_t,
    \qquad
    \mathbb E[Z_t]=\delta.
\]
Hence
\[
    \mathbb E\left[
    \#\{t\in[T]:\ibid^t\notin\mathcal B_{\v^t}\}
    \right]
    \le
    \sum_{t=1}^T\mathbb E[Z_t]
    =
    \delta T.
\]
Since $Z_1,\dots,Z_T$ are independent Bernoulli random variables with mean
$\delta$, Hoeffding's inequality implies that with probability at least
$1-\zeta$,
\[
    \sum_{t=1}^T Z_t
    \le
    \delta T+\sqrt{\frac{T\log(1/\zeta)}{2}}.
\]
This proves the claimed high-probability bound on overbidding rounds.

\noindent\textbf{Regret Analysis.} For any $\v\in\mathcal{V}$ and $\eta>0$, $\delta\in(0, 1]$ and $t\in [T]$, let
\begin{align*}
    \Phi^t_\v(\eta):=\frac{1}{\eta}\log\left(\frac{1}{|\mathscr{P}_{\v}|}\sum_{\path\in\mathscr{P}_{\v}}\exp\left(\eta \sum_{s=1}^t\widehat{\omega}^s_\v(\path)\right)\right)\,.
\end{align*}
So, for some $\eta_t\leq \frac{\delta}{\maxbid|\overline{\mathsf{E}}|}$ and $\delta\in(0, 1]$,
\begin{align*}
    &\Phi^t_\v(\eta_t) - \Phi^{t-1}_\v(\eta_t) \\
    &= \frac{1}{\eta_t}\log\left(\frac{\sum_{\path\in\mathscr{P}_{\v}}\exp(\eta_t \sum_{s=1}^{t-1}\widehat{\omega}^s_\v(\path))\cdot\exp(\eta_t\widehat{\omega}^t_\v(\path))}{\sum_{\path'\in\mathscr{P}_{\v}}\exp(\eta_t \sum_{s=1}^{t-1}\widehat{\omega}^s_\v(\path'))}\right)\\
    &\stackrel{\eqref{eq:hedge-unk}}{=}\frac{1}{\eta_t}\log\left(\sum_{\path\in\mathscr{P}_{\v}}\widetilde{\P}[\path^t=\path|\v^t=\v]\cdot\exp(\eta_t\widehat{\omega}^t_\v(\path))\right)\\
    &\stackrel{(i)}{\leq} \frac{1}{\eta_t}\log\left(\sum_{\path\in\mathscr{P}_{\v}}\widetilde{\P}[\path^t=\path|\v^t=\v]\cdot(1+\eta_t\widehat{\omega}^t_\v(\path)+\eta_t^2\widehat{\omega}^t_\v(\path)^2)\right)\\
    &=\frac{1}{\eta_t}\log\left(1+ \sum_{\path\in\mathscr{P}_{\v}}\widetilde{\P}[\path^t=\path|\v^t=\v]\cdot(\eta_t\widehat{\omega}^t_\v(\path)+\eta_t^2\widehat{\omega}^t_\v(\path)^2)\right)\\
    &\stackrel{(ii)}{\leq} \frac{1}{\eta_t}\log\left(1+\frac{\eta_t}{1-\delta}\sum_{\path\in\overline{\mathscr{P}}}\P[\path^t=\path|\v^t=\v]\widehat{\omega}^t_\v(\path)+\frac{\eta_t^2}{1-\delta}\sum_{\path\in\overline{\mathscr{P}}}\P[\path^t=\path|\v^t=\v]\widehat{\omega}^t_\v(\path)^2\right)\\
    &\stackrel{(iii)}{\leq} \frac{1}{1-\delta}\left(\sum_{\path\in\overline{\mathscr{P}}}\P[\path^t=\path|\v^t=\v]\widehat{\omega}^t_\v(\path)+\eta_t\sum_{\path\in\overline{\mathscr{P}}}\P[\path^t=\path|\v^t=\v]\widehat{\omega}^t_\v(\path)^2\right)\,.
\end{align*}
Here, $(i)$ uses $e^x\leq 1+x+x^2$ for $x\leq1$ because for $\eta_t\leq \frac{\delta}{\maxbid|\overline{\mathsf{E}}|}$, any $\v\in\mathcal{V}$ and $\path\in\mathscr{P}_{\v}$, 
\begin{align*}
   \eta_t\widehat{\omega}^t_\v(\path)=\eta_t\sum_{e\in\path}\widehat{\omega}^t_\v(e)\leq \eta_t\sum_{e\in\path}\frac{|\overline{\mathsf{E}}|}{\delta}\leq \frac{\eta_t\maxbid|\overline{\mathsf{E}}|}{\delta}\leq 1\,.
\end{align*}
where the first inequality follows from \cref{lem:properties-bandit-unknown} (2). The inequality $(ii)$ holds because, for $\path\in\mathscr P_\v$,
\[
  \widetilde{\P}[\path^t=\path|\v^t=\v]
  =
  \frac{\P[\path^t=\path|\v^t=\v]
  -\delta\cdot\textsc{Count}[\path\in\mathcal C]/|\mathcal C|}
  {1-\delta}
  \le
  \frac{\P[\path^t=\path|\v^t=\v]}{1-\delta}.
\]
Since all estimated gains are nonnegative, we may enlarge the resulting sums
from $\mathscr P_\v$ to $\overline{\mathscr P}$,
and $(iii)$ is true as $\log(1+x)\leq x, \forall x>-1$. Summing over $t=1, \dots, T$, the left hand side becomes
\begin{align*}
  \sum_{t=1}^T\left(\Phi^t_\v(\eta_t) - \Phi^{t-1}_\v(\eta_t) \right) &=  \Phi^T_\v(\eta_T) - \Phi^{0}_\v(\eta_{1})+ \sum_{t=1}^{T-1}\left(\Phi^t_\v(\eta_t) - \Phi^{t}_\v(\eta_{t+1}) \right).
\end{align*}

By definition, $\Phi^0_\v(\eta_1)=0$. Recall the definition of the policy class $\widetilde{\safepolicy}$ from \cref{eq:tilde-policy-class}. For any $\widetilde{\pi}\in\widetilde{\safepolicy}$,
\begin{align*}
   \Phi^T_\v(\eta_T) &= \frac{1}{\eta_T}\log\left(\frac{1}{|\mathscr{P}_{\v}|}\sum_{\path\in\mathscr{P}_{\v}}\exp\left(\eta_T \sum_{s=1}^T\widehat{\omega}^s_\v(\path)\right)\right)\\
   &\geq \frac{1}{\eta_T}\log\left(\frac{1}{|\mathscr{P}_{\v}|}\cdot\exp\left(\eta_T \sum_{s=1}^T\widehat{\omega}^s_\v(\widetilde{\pi}(\v))\right)\right)\\
   &=-\frac{\log |\mathscr{P}_{\v}|}{\eta_T} + \sum_{t=1}^T\widehat{\omega}^t_\v(\widetilde{\pi}(\v))\,.\numberthis \label{eq:part1}
\end{align*}
By \cref{cl:non-decreasing-func}, $\Phi^t_\v(\cdot)$ is a non-decreasing function which implies that $\Phi^t_\v(\eta_t)\geq \Phi^t_\v(\eta_{t+1})$ for $\eta_t\geq \eta_{t+1}$.\footnote{Observe that the proof of \cref{cl:non-decreasing-func} does not rely on the sign of the coefficient of $\eta$ in the exponential.} Hence, 
\begin{align*}
    \sum_{t=1}^T\left(\Phi^t_\v(\eta_t) - \Phi^{t-1}_\v(\eta_t) \right) \stackrel{\eqref{eq:part1}}{\geq}-\frac{\log |\mathscr{P}_{\v}|}{\eta_T} + \sum_{t=1}^T\widehat{\omega}^t_\v(\widetilde{\pi}(\v))\,.
\end{align*}
Thus,
\begin{align*}
  -&\frac{\log |\mathscr{P}_{\v}|}{\eta_T}
  + \sum_{t=1}^T\widehat{\omega}^t_\v(\widetilde{\pi}(\v)) \\
  &\leq
  \frac{1}{1-\delta}
  \left(
  \sum_{t=1}^T\sum_{\path\in\overline{\mathscr{P}}}
  \P[\path^t=\path\mid\v^t=\v]\widehat{\omega}^t_\v(\path)
  +
  \sum_{t=1}^T\eta_t
  \sum_{\path\in\overline{\mathscr{P}}}
  \P[\path^t=\path\mid\v^t=\v]\widehat{\omega}^t_\v(\path)^2
  \right)\\
  &\leq
  \frac{1}{1-\delta}
  \left(
  \sum_{t=1}^T\sum_{\path\in\overline{\mathscr{P}}}
  \P[\path^t=\path\mid\v^t=\v]\widehat{\omega}^t_\v(\path)
  +
  \maxbid\sum_{t=1}^T\eta_t
  \sum_{\path\in\overline{\mathscr{P}}}
  \P[\path^t=\path\mid\v^t=\v]
  \sum_{e\in\path}(\widehat{\omega}^t_\v(e))^2
  \right),
\end{align*}
where the last inequality follows from Cauchy-Schwarz:
\(
    \widehat{\omega}^t_\v(\path)^2
    =
    \left(\sum_{e\in\path}\widehat{\omega}^t_\v(e)\right)^2
    \le
    \maxbid\sum_{e\in\path}(\widehat{\omega}^t_\v(e))^2 .
\) 
Taking conditional expectations and using \cref{lem:properties-bandit-unknown}, we obtain
\begin{align*}
  -&\frac{\log |\mathscr{P}_{\v}|}{\eta_T}
  + \sum_{t=1}^T \omega^t_\v(\widetilde{\pi}(\v))\\
  &\leq
  \frac{1}{1-\delta}
  \left(
  \sum_{t=1}^T\sum_{\path\in\overline{\mathscr{P}}}
  \P[\path^t=\path\mid\v^t=\v]\omega^t_\v(\path)
  +
  \maxbid\sum_{t=1}^T\eta_t
  \sum_{\path\in\overline{\mathscr{P}}}
  \P[\path^t=\path\mid\v^t=\v]
  \sum_{e\in\path}
  \frac{|\overline{\mathsf E}|}{\delta}
  \right)\\
  &\leq
  \frac{1}{1-\delta}
  \left(
  \sum_{t=1}^T\sum_{\path\in\overline{\mathscr{P}}}
  \P[\path^t=\path\mid\v^t=\v]\omega^t_\v(\path)
  +
  \frac{\maxbid^2|\overline{\mathsf E}|}{\delta}
  \sum_{t=1}^T\eta_t
  \right).
\end{align*}
Multiplying both sides by $(1-\delta)$, rearranging, and using
$\omega^t_\v(\widetilde{\pi}(\v))\le \maxbid$, we get
\begin{align*}
  \sum_{t=1}^T{\omega}^t_\v(\widetilde{\pi}(\v))
  -\sum_{t=1}^T\sum_{\path\in\overline{\mathscr{P}}}
  \P[\path^t=\path\mid\v^t=\v]{\omega}^t_\v(\path)
  &\le
  \frac{\log |\mathscr{P}_{\v}|}{\eta_T}
  + \frac{\maxbid^2|\overline{\mathsf E}|}{\delta}\sum_{t=1}^T\eta_t
  +\maxbid T\delta .
\end{align*}
Recall that for edge
$e=(\ell-1,b_{\ell-1},s_{\ell-1})\to(\ell,b_\ell,s_\ell)$ in layer
$\ell\in[\maxbid]$,
\begin{align}
    \omega^t_\v(e)
    &=
    \frac{\alpha+(v_\ell-\alpha b_\ell)\ind{b_\ell\geq \ordotherbid{t}{\ell}}}{1+\alpha}.
    \tag{\eqref{eq:edge-weight-online-unk} restated}
\end{align}
Thus, for the bid vector $\ibid\in\bidclass$ corresponding to $\path\in\overline{\mathscr{P}}$,
\begin{align*}
    \omega^t_\v(\path)
    =
    \sum_{e\in\path}\omega^t_\v(e)
    \stackrel{\eqref{eq:utility-decomp}}{=}
    \frac{\maxbid\alpha+\ut_\v(\ibid;\otherbid{t})}{1+\alpha}.
\end{align*}
Therefore, for any $\pi\in\safepolicy$,
\begin{align*}
    \sum_{t=1}^T\ut_{\v}(\pi(\v);\otherbid{t})
    -
    \sum_{t=1}^T\sum_{\ibid\in\bidclass}
    \P[\ibid^t=\ibid\mid\v^t=\v]\ut_{\v}(\ibid;\otherbid{t})
    &\le
    (1+\alpha)
    \left(
    \frac{\log |\mathscr{P}_{\v}|}{\eta_T}
    + \frac{\maxbid^2|\overline{\mathsf E}|}{\delta}\sum_{t=1}^T\eta_t
    +\maxbid T\delta
    \right)\\
    &\lesssim
    \frac{\maxbid\log(1/\epsilon)}{\eta_T}
    + \frac{\maxbid^2|\overline{\mathsf E}|}{\delta}\sum_{t=1}^T\eta_t
    +\maxbid T\delta,
\end{align*}
where the last line uses $\alpha\le1$ and $\log|\mathscr{P}_{\v}|\lesssim \maxbid\log(1/\epsilon)$.
Taking expectations with respect to the contexts and the history, and maximizing over the policies in $\safepolicy$, we have
\begin{align*}
    \mathsf{Reg}_{nb}(T) &=\max_{\pi\in\safepolicy}\sum_{t=1}^T\E[\ut_{\v^t}(\pi(\v^t); \otherbid{t})]- \sum_{t=1}^T\E[\ut_{\v^t}(\allbids^t)]\lesssim \frac{\maxbid\log 1/\epsilon}{\eta_T}+ \frac{\maxbid^2|\overline{\mathsf{E}}|}{\delta}\sum_{t=1}^T\eta_t+\maxbid T\delta\,.
\end{align*}
Suppose for all $t\geq1$, $\eta_t=\eta=\frac{\delta^2}{\maxbid|\overline{\mathsf{E}}|}$. Then, for $T\geq T_0=:\maxbid |\overline{\mathsf{E}}|\log 1/\epsilon$, set
\begin{align*}
    \delta = \left(\frac{\maxbid |\overline{\mathsf{E}}|\log 1/\epsilon}{T}\right)^{1/3}.
\end{align*}
Note that $\eta_t\leq \frac{\delta}{\maxbid|\overline{\mathsf{E}}|}$ since $\delta\leq 1$. Recall from \cref{sec:learning-safe} that $|\overline{\mathsf{E}}|=O(\maxbid^2/\epsilon^3)$. Hence, for $T\geq T_0$,
\begin{align*}
    \mathsf{Reg}_{nb}(T) &\lesssim \frac{\maxbid^2T^{2/3}(\log 1/\epsilon)^{1/3}}{\epsilon}\,.
\end{align*}
Since regret in any round is bounded by $O(\maxbid)$, the regret in the first $T_0$ rounds is $\lesssim\maxbid T_0$. Hence, 
\begin{align*}
    \mathsf{Reg}_{nb}(T) &\lesssim \frac{\maxbid^2T^{2/3}(\log 1/\epsilon)^{1/3}}{\epsilon}+\frac{\maxbid^4\log 1/\epsilon}{\epsilon^3}\,.
\end{align*}
\textbf{Space and Time Complexity.} In this setting, the bidder alternates between two modes depending on the mixing parameter $\delta$. If the bidder samples a path from the edge path cover, the per-round time and space complexity is $O(\maxbid^2/\epsilon^3)$. If instead the bidder follows the exponential-weights updates, the complexity matches the bandit setting with known context distribution, namely $O(|\mathcal V|\maxbid^2/\epsilon^3)$. Therefore, the overall per-round time and space complexity is $O(|\mathcal V|\maxbid^2/\epsilon^3)$.

\subsection{Proof of \cref{thm:regret-LB}}
We build upon and generalise the proof of regret lower bound for FPA in \cite{han2024optimal} intended for profit maximizers~($\alpha=1$). We construct a stochastic adversary whose distribution makes it harder for the bidder to determine their optimal bid. 

Concretely, let $\mathcal{V}=\{\v\}$, $K=\maxbid$ and $\v=[1, \dots, 1]\in\R^\maxbid$. Assume ties are always resolved in favor of the bidder in consideration and the cost of capital $\alpha\in(\tfrac{1}{2}, 1]$. Define 
\begin{align}\label{eq:regret-LB-competing-bids}
    \otherbid{\clubsuit} = [c, \dots, c]\quad\text{and}\quad
    \otherbid{\diamondsuit} = [\gamma c, \dots, \gamma c],
\end{align}
where
\[
\gamma > \frac{\alpha}{2\alpha-1}, \qquad c = \frac{1}{\alpha(2\gamma-1)}\,.
\]
It is easy to verify that $\gamma>1$ and
\[
\gamma c = \frac{\gamma}{\alpha(2\gamma-1)} < 1,
\]
since $f(x)=\frac{x}{2x-1}$ is decreasing over $\gamma\in(1, \infty)$. 

Consider the following two scenarios: 

\noindent\textbf{Scenario 1.} In this scenario, for every $t\in[T]$, the competing bids $\otherbid{t}$ are:
\begin{align*}
    \otherbid{t}=\begin{cases}
        \otherbid{\clubsuit}, &\text{ w.p. } \frac{1}{2}+\delta,\\
        \otherbid{\diamondsuit}, &\text{ w.p. } \frac{1}{2}-\delta,\\
    \end{cases}
\end{align*}

\noindent\textbf{Scenario 2.} In this scenario, for every $t\in[T]$, the competing bids $\otherbid{t}$ are:
\begin{align*}
    \otherbid{t}=\begin{cases}
        \otherbid{\clubsuit}, &\text{ w.p. } \frac{1}{2}-\delta,\\
        \otherbid{\diamondsuit}, &\text{ w.p. } \frac{1}{2}+\delta,\\
    \end{cases}
\end{align*}

for some $\delta\in(0, \tfrac{1}{4})$ to be determined shortly. Assume the randomness used in different rounds are independent. Then, for $\delta\in(0, \frac{1}{4})$,
\[
    \textsc{KL}(P||Q) = T\cdot\textsc{KL}(\textsc{Bern}(0.5+\delta)||\textsc{Bern}(0.5-\delta))=2T\delta\log\Big(\frac{1+2\delta}{1-2\delta}\Big)\leq \frac{8T\delta^2}{1-2\delta}\leq 16T\delta^2
\]
where the first inequality follows from $\log(\frac{1+x}{1-x})\leq \frac{2x}{1-x}$. By \citet[Lemma 2.6]{tsybakov2009introduction},
\[
    1-\textsc{TV}(P, Q) \geq \frac{1}{2}\exp{(-\textsc{KL}(P||Q))}\geq \frac{1}{2}\exp{\left(-16T\delta^2\right)}\,.
\]

\noindent Now, we inspect the separation between the two scenarios. Observe that
\begin{align*}
    \max_{\ibid\in\bidclass_{\v}}\E_P[\ut_{\v}(\ibid; \otherbid{t})] &\geq \E_P[\ut_{\v}([c, \dots, c]; \otherbid{t})] = \Big(\frac{1}{2}+\delta\Big)\maxbid(1-\alpha c),\\
    \max_{\ibid\in\bidclass_{\v}}\E_Q[\ut_{\v}(\ibid; \otherbid{t})] &\geq \E_Q[\ut_{\v}([\gamma c, \dots, \gamma c]; \otherbid{t})] = \maxbid(1-\gamma\alpha c).
\end{align*}
Now, we aim to compute $\max_{\ibid\in\bidclass_{\v}}\E_{(P+Q)/2}[\ut_{\v}(\ibid; \otherbid{t})]$. Since $\v=[1, \dots, 1]$,
\begin{align*}
   \max_{\ibid\in\bidclass_{\v}}\E_{(P+Q)/2}[\ut_{\v}(\ibid; \otherbid{t})] &\leq \max_{\ibid: 1\geq b_1\geq\cdots\geq b_\maxbid\geq0}\E_{(P+Q)/2}[\ut_{\v}(\ibid; \otherbid{t})] \\
   &=\frac{1}{2}\Big(\sum_{i=1}^\maxbid (1-\alpha b_i)\cdot(\ind{b_i\geq \gamma c}+\ind{b_i\geq c})\Big)\,.
\end{align*}
Thus, we can only consider threshold strategies of the form \(\ibid_j=[\gamma c,\dots,\gamma c, c, \dots, c]\) that contain $j$ entries of $\gamma c$ and $\maxbid-j$ entries of $c$ for $j\in\{0,\dots,\maxbid\}$. For any $j\in\{0, \dots, \maxbid\}$,
\begin{align*}
    \E_{(P+Q)/2}[\ut_{\v}(\ibid_j; \otherbid{t})] =\frac{1}{2}\left( \maxbid(1-\alpha c) + j(1-\alpha c(2\gamma-1))\right) = \frac{1}{2}\maxbid(1-\alpha c),
\end{align*}
since $\alpha c(2\gamma-1)=1$. Hence, $\max_{\ibid\in\bidclass_{\v}}\E_{(P+Q)/2}[\ut_{\v}(\ibid; \otherbid{t})]\leq \frac{1}{2}\maxbid(1-\alpha c)$. For any $\ibid\in\bidclass_\v$,
\begin{align*}
    &\max_{\ibid^*\in\bidclass_{\v}}\E_{P}[\ut_{\v}(\ibid^*; \otherbid{t})-\ut_{\v}(\ibid; \otherbid{t})] + \max_{\ibid^*\in\bidclass_{\v}}\E_{Q}[\ut_{\v}(\ibid^*; \otherbid{t})-\ut_{\v}(\ibid; \otherbid{t})]\\
    &\geq \max_{\ibid^*\in\bidclass_{\v}}\E_{P}[\ut_{\v}(\ibid^*; \otherbid{t})] + \max_{\ibid^*\in\bidclass_{\v}}\E_{Q}[\ut_{\v}(\ibid^*; \otherbid{t})]-2\max_{\ibid^*\in\bidclass_\v}\E_{(P+Q)/2}[\ut_{\v}(\ibid^*; \otherbid{t})]\\
    &=\Big(\frac{1}{2}+\delta\Big)\maxbid(1-\alpha c) + \maxbid(1-\gamma\alpha c) - \maxbid(1-\alpha c)=\maxbid(1-\alpha c)\delta=\frac{2\maxbid\delta(\gamma-1)}{2\gamma-1},
\end{align*}
where the last line uses $\alpha c(2\gamma-1)=1$. 

Thus, any $\ibid\in\bidclass_\v$ incurs a total regret of $\frac{\maxbid T\delta(\gamma-1)}{2\gamma-1}$ under $P$~(Scenario 1), or a total regret of $\frac{\maxbid T\delta(\gamma-1)}{2\gamma-1}$ under $Q$~(Scenario 2). By two-point method from~\citet[Theorem 2.2]{tsybakov2009introduction},
\begin{align*}
    \E_{(P+Q)/2}[\textsf{Reg}_{nb}(T)] \geq \frac{\maxbid T\delta(\gamma-1)}{2\gamma-1}\cdot(1-\text{TV}(P, Q))\geq \frac{\maxbid T\delta(\gamma-1)}{2(2\gamma-1)}\exp{\left(-16T\delta^2\right)}
\end{align*}

Setting $\delta=\frac{1}{4\sqrt{2T}}$, we get $\E_{(P+Q)/2}[\textsf{Reg}_{nb}(T)] =\Omega(\maxbid\sqrt{T})$.

\subsection{Proof of \cref{thm:regret-with-budgets}}\label{apx:thm:regret-with-budgets}
Before analyzing regret, we first state the necessary changes in the bandit setting.

\noindent\textbf{Known context distribution.} The algorithm is identical to \cref{alg:primal-dual}, except that edge weights $\omega^t_\v(e)$ are replaced with their unbiased estimators from \cref{eq:w-hat-estimate-exp3}, where $\omega^t_\v(e)$ is defined in \cref{eq:edge-weight-online-PD}.

\noindent\textbf{Unknown context distribution.} For any $\v\in\mathcal{V}$, round $t$, and edge $e=(\l-1,b_{\l-1},s_{\l-1})\to(\l,b_\ell,s_\ell)$ in layer $\ell\in[\maxbid]$, the edge weight is
\begin{align}\label{eq:edge-weight-online-PD-2}
\omega^t_\v(e) = \frac{\alpha+1/\rho+(v_\ell-(\alpha+\lambda_t)b_\ell)\cdot\ind{b_\ell\geq \ordotherbid{t}{\ell}}}{1+\alpha+1/\rho},
\end{align}
with all edges from layer $M$ to $d$ assigned weight $0$. Here $\lambda_t\in[0,1/\rho]$ is the dual variable, ensuring $\omega^t_\v(e)\in[0,1]$. As earlier, we add an $s$-$d$ edge for the null action $\emptyset$, with weight $\maxbid(\alpha+1/\rho)/(1+\alpha+1/\rho)$. In this case, \ARef*{line:PD-primal} is replaced as follows: Submit $\ibid^t\in\bidclass$ via \ARef*{line:map-unk-start} to \ARef**{line:map-unk-end}. Finally, the edge weights are replaced with their estimators in \cref{eq:w-hat-estimate-exp3-unk} with $\omega^t_\v(e)$ per \cref{eq:edge-weight-online-PD-2}.

\noindent\textbf{Regret Analysis.} We begin with the following result about the primal regret minimizer:
\begin{lemma}
    For any sequence of dual variables $\{\lambda_t\}_{t\ge1}$ with
    $\lambda_t\in[0,1/\rho]$, we have
    \begin{align}\label{eq:primal-UB}
        \max_{\pi\in\safepolicy\cup\bot}
        \sum_{t=1}^T\E[\ut_{\v^t}(\pi(\v^t); \otherbid{t})-\lambda_t\price(\pi(\v^t); \otherbid{t})]
        - \sum_{t=1}^T\E[\ut_{\v^t}(\allbids^t)-\lambda_t\price(\allbids^t)]
        \lesssim \mathsf{R}_P^T ,
    \end{align}
    where $\allbids^t=(\ibid^t; \otherbid{t})$, $\bot$ is the null policy, and $\mathsf{R}_P^T$ is
    \begin{itemize}
    \item $\frac{\maxbid^{3/2}}{\rho}\sqrt{T\log 1/\epsilon}$ in the full-information setting;
    \item $\frac{\maxbid^2}{\rho\epsilon^{3/2}}\sqrt{T\log 1/\epsilon}$ in the bandit setting with known context distribution;
    \item $\frac{\maxbid^2 T^{2/3}(\log 1/\epsilon)^{1/3}}{\rho\epsilon} + \frac{\maxbid^4\log 1/\epsilon}{\rho\epsilon^3}$ in the bandit setting with unknown context distribution.
    \end{itemize}
\end{lemma}
\begin{proof}
We prove the result for the full information setting. The results in the other settings follows accordingly. Following the analysis of \cref{thm:regret-safe} up to \cref{eq:regret-stoc-in-path}, we get
\begin{align*}
        \sum_{t=1}^T\sum_{\path\in\mathscr{P}_\v}\P[\path^t=\path|\v^t=\v]\omega^t_\v(\path) - \sum_{t=1}^T\omega^t_\v(\widetilde{\pi}(\v))&\lesssim \frac{\maxbid\log 1/\epsilon}{\eta_T} + \sum_{t=1}^T\sum_{\path\in\mathscr{P}_\v}\P[\path^t=\path|\v^t=\v]\eta_t\omega^t_\v(\path)^2\,. 
    \end{align*}
    Using the definition of edge weights in \cref{eq:edge-weight-online-PD}, we get
    \begin{align*}
        \omega^t_\v(\path) &=\sum_{e\in\path}\omega^t_\v(e)=\frac{M-(\ut_{\v}(\ibid; \otherbid{t})-\lambda_t\price(\ibid; \otherbid{t}))}{1+\alpha+1/\rho}\,.
    \end{align*}
    Then, continuing with analysis of \cref{thm:regret-safe} and using the fact that $1+\alpha+1/\rho\leq 3/\rho$ since $\alpha\leq 1$ and $\rho<1$, we get the desired regret bound.
\end{proof}

Furthermore, the OGD algorithm~(dual regret minimizer) satisfies the following regret bound:
\begin{proposition}[\cite{hazan2016introduction}]
    Let $\zeta_t>0, \forall t\geq 1$ be the learning rate in round $t$. For any fixed $\lambda\in[0, \frac{1}{\rho}]$,
    \begin{align}\label{eq:dual-UB}
        \mathsf{Reg}_D(T):=\sum_{t=1}^T(\price(\allbids^t)-\rho\maxbid)(\lambda-\lambda_t)\lesssim \frac{1}{\zeta_T\rho^2}+\maxbid^2\sum_{t=1}^T\zeta_t\,.
    \end{align}
    where $\allbids^t=(\ibid^t; \otherbid{t}), \forall t\geq 1$. Setting $\zeta_t=\frac{1}{\rho\maxbid\sqrt{t}}$ yields $\mathsf{Reg}_D(T)\lesssim \frac{M\sqrt{T}}{\rho}$.
\end{proposition}

In what follows, we conduct the analysis for the full-information setting and the bandit setting with known context distribution. The bandit setting with unknown context distribution follows from the same argument with minor modifications at the end of the proof.

Let $\tau=\min\{t\in[T]: B_t<\maxbid\}$, the stopping time of \cref{alg:primal-dual}. If no such round exists, define $\tau=T$. Let the primal regret upper bound~(the right hand side of \cref{eq:primal-UB}) up to $\tau$ be denoted as $\mathsf{R}_P^\tau$. Rearranging \cref{eq:primal-UB} we get,
\begin{align}\label{eq:primal-regret}
    \sum_{t=1}^\tau \E[\ut_{\v^t}(\allbids^t)] \gtrsim \max_{\pi\in\safepolicy\cup\bot}\sum_{t=1}^\tau\E[\ut_{\v^t}(\pi(\v^t); \otherbid{t})-\lambda_t\price(\pi(\v^t); \otherbid{t})] + \sum_{t=1}^\tau \E[\lambda_t\price(\allbids^t)] - \mathsf{R}_P^\tau.
\end{align}
Similarly, define the OGD regret bound up to round $\tau$ as $\mathsf{R}_D^\tau=\frac{\maxbid\sqrt{\tau}}{\rho}$. Rearranging \cref{eq:dual-UB}, we get
\begin{align}
    \sum_{t=1}^\tau\lambda_t \price(\allbids^t) \gtrsim \sum_{t=1}^\tau \rho\maxbid(\lambda_t-\lambda) + \sum_{t=1}^\tau \lambda \price(\allbids^t) - \mathsf{R}_D^\tau\,.
\end{align}
Substituting in \cref{eq:primal-regret}, we get
\begin{align}\label{eq:first-step}
    \sum_{t=1}^\tau \E[\ut_{\v^t}(\allbids^t)] &\gtrsim \max_{\pi\in\safepolicy\cup\bot}\sum_{t=1}^\tau\E[\ut_{\v^t}(\pi(\v^t); \otherbid{t})-\lambda_t\price(\pi(\v^t); \otherbid{t})]\nonumber\\
    &\quad+ \sum_{t=1}^\tau \E[\rho\maxbid(\lambda_t-\lambda)] + \sum_{t=1}^\tau \lambda \E[\price(\allbids^t)] - \mathsf{R}_D^\tau - \mathsf{R}_P^\tau\nonumber\\
    &=\max_{\pi\in\safepolicy\cup\bot}\sum_{t=1}^\tau\E[\ut_{\v^t}(\pi(\v^t); \otherbid{t}) + \lambda_t(\rho\maxbid- \price(\pi(\v^t); \otherbid{t}))]\nonumber\\
    &\quad+ \sum_{t=1}^\tau \lambda \E[\price(\allbids^t)] - \mathsf{R}_D^\tau - \mathsf{R}_P^\tau-\lambda\rho\maxbid \tau\,.
\end{align}

Define
\begin{align}\label{eq:opt-nb-def}
    \mathsf{OPT}_{nb}^\tau&=\max_{\pi\in\safepolicy}\sum_{t=1}^\tau\E[\ut_{\v^t}(\pi(\v^t); \otherbid{t})]\\
    \pi^*&=\argmax_{\pi\in\safepolicy}\sum_{t=1}^\tau\E[\ut_{\v^t}(\pi(\v^t); \otherbid{t})]\label{eq:pi-star-def}
\end{align}
 In words, $\mathsf{OPT}_{nb}^\tau$ is the maximum utility obtained by a stationary policy in $\safepolicy$ over first $\tau$ rounds without the budget constraints and $\pi^*$ is the stationary policy achieving it. Then, 
 
\begin{claim}\label{eq:main-LB}
For any sequence of competing bids $\{\otherbid{t}\}_{t\geq 1}$, dual variables $\{\lambda_t\}_{t\geq 1}$ and $\rho<1$, we have
  \begin{align*}
    \max_{\pi\in\safepolicy\cup\bot}\sum_{t=1}^\tau\E[\ut_{\v^t}(\pi(\v^t); \otherbid{t}) + \lambda_t(\rho\maxbid- \price(\pi(\v^t); \otherbid{t}))] \geq \rho\cdot\mathsf{OPT}_{nb}^\tau\,.
\end{align*}  
where $\mathsf{OPT}_{nb}^\tau$ is defined per \cref{eq:opt-nb-def}.
\end{claim}
\begin{proof}
    By definition,
    \begin{align}\label{eq:main-pivot}
&\max_{\pi\in\safepolicy\cup\bot}\sum_{t=1}^\tau\E[\ut_{\v^t}(\pi(\v^t); \otherbid{t}) + \lambda_t(\rho\maxbid- \price(\pi(\v^t); \otherbid{t}))] \nonumber\\
&\quad\geq \max_{\pi\in\{\pi^*, \bot\}}\sum_{t=1}^\tau\E[\ut_{\v^t}(\pi(\v^t); \otherbid{t}) + \lambda_t(\rho\maxbid- \price(\pi(\v^t); \otherbid{t}))],
\end{align}
where $\pi^*$ is per \cref{eq:pi-star-def}. Now, consider two cases:
\begin{enumerate}
    \item Suppose $\sum_{t=1}^\tau\E[\ut_{\v^t}(\pi^*(\v^t); \otherbid{t}) -\lambda_t\price(\pi^*(\v^t); \otherbid{t}))]\geq 0$. Then,
    \begin{align*}
        &\max_{\pi\in\safepolicy\cup\bot}\sum_{t=1}^\tau\E[\ut_{\v^t}(\pi(\v^t); \otherbid{t}) + \lambda_t(\rho\maxbid- \price(\pi(\v^t); \otherbid{t}))] \\
        &\quad\geq \sum_{t=1}^\tau\E[\ut_{\v^t}(\pi^*(\v^t); \otherbid{t}) + \lambda_t(\rho\maxbid- \price(\pi^*(\v^t); \otherbid{t}))] \\
        &\quad\geq \sum_{t=1}^\tau\E[\ut_{\v^t}(\pi^*(\v^t); \otherbid{t}) -(1-\rho) \lambda_t\price(\pi^*(\v^t); \otherbid{t})] \\
        &\quad\geq \sum_{t=1}^\tau\E[\ut_{\v^t}(\pi^*(\v^t); \otherbid{t}) -(1-\rho) \ut_{\v^t}(\pi^*(\v^t); \otherbid{t})] \\
        &=\rho\sum_{t=1}^\tau\E[\ut_{\v^t}(\pi^*(\v^t); \otherbid{t})]= \rho\cdot\mathsf{OPT}^\tau_{nb}\,.
    \end{align*}
    Here, the second inequality holds as $\price(\pi^*(\v^t); \otherbid{t})\leq \maxbid$ and third inequality holds because by assumption, $\sum_{t=1}^\tau\E[\ut_{\v^t}(\pi^*(\v^t); \otherbid{t})] \geq\sum_{t=1}^\tau\E[\lambda_t\price(\pi^*(\v^t); \otherbid{t}))]$.

    \item Suppose $\sum_{t=1}^\tau\E[\ut_{\v^t}(\pi^*(\v^t); \otherbid{t}) -\lambda_t\price(\pi^*(\v^t); \otherbid{t}))]< 0$. Setting $\pi=\bot$ in the right hand side of \cref{eq:main-pivot},
    \begin{align*}
        \max_{\pi\in\safepolicy\cup\bot}\sum_{t=1}^\tau\E[\ut_{\v^t}(\pi(\v^t); \otherbid{t}) + \lambda_t(\rho\maxbid- \price(\pi(\v^t); \otherbid{t}))] &\geq \rho\maxbid \sum_{t=1}^\tau \E[\lambda_t] \\
        &\geq \rho\sum_{t=1}^\tau \E[\lambda_t\price(\pi^*(\v^t); \otherbid{t})]\\
        &\geq \rho\sum_{t=1}^\tau \E[\ut_{\v^t}(\pi^*(\v^t); \otherbid{t})]\\
        &=\rho\cdot\mathsf{OPT}_{nb}^\tau\,.
    \end{align*}
\end{enumerate}
\end{proof}

Substituting the result of \cref{eq:main-LB} in \cref{eq:first-step} and rearranging, we get
\begin{align*}
  \sum_{t=1}^\tau \E[\ut_{\v^t}(\allbids^t)-\lambda \price(\allbids^t)] &\gtrsim  \rho\cdot\mathsf{OPT}_{nb}^\tau - \mathsf{R}_D^\tau - \mathsf{R}_P^\tau-\lambda\rho\maxbid \tau
\end{align*}
\begin{claim}\label{cl:opt-LB}
    Recall that $\mathsf{OPT}$ is the optimal value of \eqref{eq:opt-sto}. Then, $\mathsf{OPT}_{nb}^\tau \geq \mathsf{OPT} -\maxbid(T-\tau)$.
\end{claim}
\begin{proof}
With slight abuse of notation, define 
\begin{align*}
     \mathsf{OPT}^s:=\max_{\pi\in\safepolicy\cup\bot}&\sum_{t=1}^{s} \E[\ut_{\v^t}(\pi(\v^t); \otherbid{t})]\quad\text{such that}\quad \sum_{t=1}^s \price(\pi(\v^t); \otherbid{t}) \leq \rho \maxbid T\,.
\end{align*}
Observe that $\mathsf{OPT}_{nb}^\tau \geq \mathsf{OPT}^\tau$ because $\mathsf{OPT}^\tau$ is the optimal objective value of the constrained problem whereas $\mathsf{OPT}_{nb}^\tau$ is the optimal objective value of the unconstrained problem. 


Let $\pi^\dagger\in\safepolicy\cup\bot$ be the policy that maximizes \eqref{eq:opt-sto}. Then,
\begin{align*}
    \mathsf{OPT} &= \sum_{t=1}^\tau\E[\ut_{\v^t}(\pi^\dagger(\v^t); \otherbid{t})] + \sum_{t=\tau+1}^T \E[\ut_{\v^t}(\pi^\dagger(\v^t); \otherbid{t})]\quad\text{such that}\quad \sum_{t=1}^T\price(\pi^\dagger(\v^t); \otherbid{t}) \leq \rho \maxbid T\\
    &\leq \sum_{t=1}^\tau\E[\ut_{\v^t}(\pi^\dagger(\v^t); \otherbid{t})] + \maxbid(T-\tau)\quad\text{such that}\quad \sum_{t=1}^T\price(\pi^\dagger(\v^t); \otherbid{t}) \leq \rho \maxbid T\\
    &\leq \max_{\pi\in\safepolicy\cup\bot}\sum_{t=1}^\tau\E[\ut_{\v^t}(\pi(\v^t); \otherbid{t})] + \maxbid(T-\tau)\quad\text{such that}\quad \sum_{t=1}^\tau \price(\pi(\v^t); \otherbid{t}) \leq \rho \maxbid T\\
    &=\mathsf{OPT}^\tau + \maxbid(T-\tau)\,.
\end{align*}
Hence, $\mathsf{OPT}^\tau \geq \mathsf{OPT} -\maxbid(T-\tau)\implies \mathsf{OPT}_{nb}^\tau \geq \mathsf{OPT} -\maxbid(T-\tau)$. 
\end{proof}

Hence, by \cref{cl:opt-LB}, we have
\begin{align*}
   \sum_{t=1}^T\E[\ut_{\v^t}(\allbids^t)] \geq \sum_{t=1}^\tau\E[\ut_{\v^t}(\allbids^t)] &\gtrsim  \rho\cdot(\mathsf{OPT} -\maxbid(T-\tau))+ \sum_{t=1}^\tau \lambda \E[\price(\allbids^t)] - \mathsf{R}_D^\tau - \mathsf{R}_P^\tau-\lambda\rho\maxbid \tau\,. 
\end{align*}
If $\tau=T$, set $\lambda=0$. Rearranging the terms and substituting the value of $\mathsf{R}_D^T$ yields
\begin{align}\label{eq:final-res1}
\rho\cdot\mathsf{Reg}(T) = \rho\cdot\mathsf{OPT} -  \sum_{t=1}^T\E[\ut_{\v^t}(\allbids^t)] \lesssim \frac{\maxbid\sqrt{T}}{\rho} + \mathsf{R}_P^T.
\end{align}
 On the other hand, if $\tau< T$, $\sum_{t=1}^\tau \price(\allbids^t) + \maxbid \geq \rho\maxbid T$. So, 
\begin{align*}
   \sum_{t=1}^T\E[\ut_{\v^t}(\allbids^t)] &\gtrsim  \rho\cdot(\mathsf{OPT} -\maxbid(T-\tau))+ \lambda(\rho\maxbid T-\maxbid) - \mathsf{R}_D^\tau - \mathsf{R}_P^\tau-\lambda\rho\maxbid \tau. 
\end{align*}
Setting $\lambda=\frac{1}{\rho}$,
\begin{align*}
   \sum_{t=1}^T\E[\ut_{\v^t}(\allbids^t)] &\gtrsim  \rho\cdot(\mathsf{OPT} -\maxbid(T-\tau))+\maxbid(T-\tau)-\frac{\maxbid}{\rho}- \mathsf{R}_D^\tau - \mathsf{R}_P^\tau.
\end{align*}
Rearranging, 
\begin{align*}
   \rho\cdot\mathsf{Reg}(T)=\rho\cdot\mathsf{OPT} -\sum_{t=1}^T\E[\ut_{\v^t}(\allbids^t)] &\lesssim  \rho\maxbid(T-\tau)-\maxbid(T-\tau)+\frac{\maxbid}{\rho}+ \mathsf{R}_D^\tau + \mathsf{R}_P^\tau\\
   \implies\rho\cdot\mathsf{Reg}(T)&\lesssim \frac{\maxbid}{\rho}+ \frac{\maxbid\sqrt{T}}{\rho} + \mathsf{R}_P^T,
\end{align*}
where the last inequality follows as $\rho\leq 1$ and $\tau\leq T$. Combining with \cref{eq:final-res1}, we get the desired regret bound.

\subsection{Proof of \cref{thm:regret-safe-alt}}

The key idea of the proof is to show that \cref{alg:weight-pushing-stoc-alt} is an efficient implementation of Decreasing Hedge (equivalently, \cref{alg:weight-pushing-stoc}) and, crucially, that its time and space complexity are independent of the number of contexts.

Recall that in the full-information setting, for each $\v\in\mathcal V$ and each edge
$e=(\l-1,b_{\l-1},s_{\l-1})\to(\l,b_\l,s_\l)$ in layer $\l\in[\maxbid]$, the edge weight is of the form 
\[\omega_\v^t(e)=x^t(e)\,v_\l+y^t(e)\]
and all edges from nodes in layer $\maxbid$ to the destination node $d$ have weight $0$. Here,
\[
x^t(e)=-\frac{\ind{b_\l\ge \ordotherbid{t}{\l}}}{1+\alpha},
\qquad
y^t(e)=\frac{1+\alpha b_\l\cdot\ind{b_\l\ge \ordotherbid{t}{\l}}}{1+\alpha}.
\]
We further defined 
\[
\widehat{x}^{t}(e):=-\eta_{t+1}\sum_{s=1}^{t}x^s(e),
\qquad
\widehat{y}^{t}(e):=-\eta_{t+1}\sum_{s=1}^{t}y^s(e),
\]
where $\widehat{x}^{0}(e)=\widehat{y}^{0}(e)=0$. This gives
\begin{align}
\widehat{x}^{t}(e) &= \frac{\eta_{t+1}}{\eta_t}\,\widehat{x}^{t-1}(e)-\eta_{t+1}x^t(e), \tag{\eqref{eq:main-iteration1} restated}\\
\widehat{y}^{t}(e) &= \frac{\eta_{t+1}}{\eta_t}\,\widehat{y}^{t-1}(e)-\eta_{t+1}y^t(e). \tag{\eqref{eq:main-iteration2} restated}
\end{align}
Recall that $\omega^s_\v(\path)=\sum_{e\in\path}\omega^s_\v(e)$. Then, for any $\path=s\to (1, b_1, s_1)\to\dots\to (\maxbid, b_\maxbid, s_\maxbid)\to d$,
\begin{align}\label{eq:cumu-loss}
   -\eta_t\sum_{s=1}^{t-1}\omega^s_\v(\path) = \langle\widehat{\x}^{t-1}(\path), \v\rangle+\widehat{y}^{t-1}(\path), 
\end{align}
where $\widehat{\x}^{t-1}(\path)\in \R^\maxbid$ and the $\l^{th}$ coordinate corresponding to $e=(\l-1, b_{\l-1}, s_{\l-1})\to(\l, b_{\l}, s_{\l})$ is 
\begin{align*}
    [\widehat{\x}^{t-1}(\path)]_\l = \widehat{x}^{t-1}(e) , \quad\text{and}\quad \widehat{y}^{t-1}(\path)=\sum_{e\in\path} \widehat{y}^{t-1}(e)\,.
\end{align*}
Recall that in round $t$, the probability of selecting a path $\path$ is $\P_\v^t(\path)=\prod_{e\in\path}\phi^t_{\v}(e),$ and for any edge $e=(\l-1, b_{\l-1}, s_{\l-1})\to (\l, b_{\l}, s_{\l}) \in \mathsf{E}_\v$, the edge probabilities are
\begin{align}
     \phi^t_\v(e) = \exp(\widehat{x}^{t-1}(e)\cdot v_\l + \widehat{y}^{t-1}(e)) \cdot \frac{\Gamma^{t-1}_\v(\l, b_{\l}, s_{\l})}{\Gamma^{t-1}_\v(\l-1, b_{\l-1}, s_{\l-1})},\tag{\eqref{eq:phi-update-alt} restated}
\end{align}
where $\Gamma^{t-1}_\v(d)=1$ and for $\Gamma^{t-1}_\v(\cdot)$ is computed recursively in a bottom-to-top fashion for every $u=(\l-1, b_{\l-1}, s_{\l-1})\in\mathsf{N}_\v$:
\begin{align}
\Gamma^{t-1}_\v(u)=\sum_{v=(\l, b_\l, s_\l):u\to v=:e\in\mathsf{E}_\v}\Gamma^{t-1}_\v(v)\cdot\exp(\widehat{x}^{t-1}(e)\cdot v_\l + \widehat{y}^{t-1}(e))\,.\tag{\eqref{eq:gamma-update-alt} restated}
\end{align}
Here, $\eta_0=1$ and $\gamma_t=\frac{\eta_t}{\eta_{t-1}}, \forall t\geq 1$. Recall that in a na\"ive implementation of the Decreasing Hedge algorithm, the probability of selecting path $\path$ in round $t$ is
\begin{align}
    \HP^{t}_\v(\path) = \frac{\exp(-\eta_t\sum_{s=1}^{t-1}\omega^s_\v(\path))}{\sum_{\path'\in\mathscr{P}_\v}\exp(-\eta_t\sum_{s=1}^{t-1}\omega^s_\v(\path'))},\tag{\eqref{eq:dec-hedge-naive} restated}
\end{align}
where $\mathscr{P}_\v$ is the set of all $s$-$d$ paths in $\DAG{}$. To show that \cref{alg:weight-pushing-stoc-alt} is an efficient implementation of Decreasing Hedge, we prove that $\P^t_\v(\path)=\HP^{t}_\v(\path)$. Since $\P^t_\v(\path)=\prod_{e\in\path}\phi^t_\v(e)$ and 
\begin{align*}
   \HP^{t}_\v(\path)
   \stackrel{\eqref{eq:dec-hedge-naive}}{=}
   \frac{\exp(-\eta_t\sum_{s=1}^{t-1}\omega^s_\v(\path))}{\sum_{\path'\in\mathscr{P}_\v}\exp(-\eta_t\sum_{s=1}^{t-1}\omega^s_\v(\path'))}
   \stackrel{\eqref{eq:cumu-loss}}{=} \frac{\exp(\langle\widehat{\x}^{t-1}(\path), \v\rangle+\widehat{y}^{t-1}(\path))}{\sum_{\path'\in\mathscr{P}_\v}\exp(\langle\widehat{\x}^{t-1}(\path'), \v\rangle+\widehat{y}^{t-1}(\path'))},
\end{align*}
it suffices to show that 
\begin{align*}
   \prod_{e\in\path}\phi^t_\v(e) = \frac{\exp(\langle\widehat{\x}^{t-1}(\path), \v\rangle+\widehat{y}^{t-1}(\path))}{\sum_{\path'\in\mathscr{P}_\v}\exp(\langle\widehat{\x}^{t-1}(\path'), \v\rangle+\widehat{y}^{t-1}(\path'))}\,.
\end{align*}
To this end, we first show the following result (similar to \cref{cl:gamma-closed-form}):
\begin{claim}\label{cl:gamma-closed-form-alt}
    For any node $u=(\l-1, b_{\l-1}, s_{\l-1})\in\mathsf{N}_\v$, let $\mathscr{P}_\v(u)$ be the set of paths starting at $u$ and terminating in $d$. Then, 
    \begin{align*}
        \Gamma^{t-1}_\v(u)=\sum_{\path\in\mathscr{P}_\v(u)}\prod_{e\in\path}\exp(\widehat{x}^{t-1}(e)\cdot v_k + \widehat{y}^{t-1}(e))\,.
    \end{align*}
    where $e=(k-1, b_{k-1}, s_{k-1})\to (k, b_k, s_k)$ is an edge between layer $k-1$ and layer $k$ for $\l\leq k\leq \maxbid$.
\end{claim}
\begin{proof}
    We prove the result by backward induction. For the base case, $\Gamma^{t-1}_\v(d)=1$. Suppose the result holds true for all the nodes in layer $\l$ for some $\l\in[\maxbid+1]$. Then, for any node $v=(\l, b_{\l}, s_{\l})$ in layer $\l$, we have
    \begin{align}\label{eq:induction-hypo}
        \Gamma^{t-1}_\v(v)=\sum_{\path\in\mathscr{P}_\v(v)}\prod_{e'\in\path}\exp(\widehat{x}^{t-1}(e')\cdot v_{j} + \widehat{y}^{t-1}(e'))\,.
    \end{align}
    where $e=(j-1, b_{j-1}, s_{j-1})\to (j, b_j, s_j)$ is an edge between layer $j-1$ and layer $j$ for $\l+1\leq j\leq \maxbid$. Thus, for $u=(\l-1, b_{\l-1}, s_{\l-1})$ in layer $\l-1$, 
    \begin{align*}
        \Gamma^{t-1}_\v(u)&\stackrel{\eqref{eq:gamma-update-alt}}{=}\sum_{v=(\l, b_\l, s_\l):u\to v=:e\in\mathsf{E}_\v}\Gamma^{t-1}_\v(v)\cdot\exp(\widehat{x}^{t-1}(e)\cdot v_\l + \widehat{y}^{t-1}(e))\\
        &\stackrel{\eqref{eq:induction-hypo}}{=}\sum_{v=(\l, b_\l, s_\l):u\to v=:e\in\mathsf{E}_\v}\Big(\sum_{\path\in\mathscr{P}_\v(v)}\prod_{e'\in\path} \exp(\widehat{x}^{t-1}(e')\cdot v_{j} + \widehat{y}^{t-1}(e'))\Big)\cdot\exp(\widehat{x}^{t-1}(e)\cdot v_\l + \widehat{y}^{t-1}(e))\\
        &=\sum_{\path\in\mathscr{P}_\v(u)}\prod_{e\in\path}\exp(\widehat{x}^{t-1}(e)\cdot v_k + \widehat{y}^{t-1}(e))\,.
    \end{align*}
\end{proof}
Hence, the probability of sampling path $\path$ in round $t$ when the context is $\v$ is
\begin{align*}
    \prod_{e\in\path}\phi^t_\v(e)&\stackrel{\eqref{eq:phi-update-alt}}{=}\prod_{e=(\l-1, b_{\l-1}, s_{\l-1})\to (\l, b_{\l}, s_{\l})\in\path} \exp(\widehat{x}^{t-1}(e)\cdot v_\l + \widehat{y}^{t-1}(e)) \cdot \frac{\Gamma^{t-1}_\v(\l, b_{\l}, s_{\l})}{\Gamma^{t-1}_\v(\l-1, b_{\l-1}, s_{\l-1})}\\
    &=\frac{\Gamma^{t-1}_\v(d)}{\Gamma^{t-1}_\v(s)}\prod_{e=(\l-1, b_{\l-1}, s_{\l-1})\to (\l, b_{\l}, s_{\l})\in\path} \exp(\widehat{x}^{t-1}(e)\cdot v_\l + \widehat{y}^{t-1}(e))\\
    &=\frac{\exp(\langle\widehat{\x}^{t-1}(\path), \v\rangle+\widehat{y}^{t-1}(\path))}{\Gamma^{t-1}_\v(s)},
\end{align*}
where the last inequality follows from the definition of $\widehat{\x}^{t-1}(\path), \widehat{y}^{t-1}(\path)$ and the fact that $\Gamma^{t-1}_\v(d)=1$. Finally, by \cref{cl:gamma-closed-form}, we get that
\begin{align*}
   \Gamma^{t-1}_\v(s)=\sum_{\path'\in\mathscr{P}_\v}\prod_{e\in\path}\exp(\widehat{x}^{t-1}(e)\cdot v_\l + \widehat{y}^{t-1}(e))=\sum_{\path'\in\mathscr{P}_\v}\exp(\langle\widehat{\x}^{t-1}(\path'), \v\rangle+\widehat{y}^{t-1}(\path'))\,. 
\end{align*}
Thus, 
\begin{align*}
   \prod_{e\in\path}\phi^t_\v(e)= \frac{\exp(\langle\widehat{\x}^{t-1}(\path), \v\rangle+\widehat{y}^{t-1}(\path))}{\sum_{\path'\in\mathscr{P}_\v}\exp(\langle\widehat{\x}^{t-1}(\path'), \v\rangle+\widehat{y}^{t-1}(\path'))},
\end{align*}
which was the desired result. 


\textbf{Space and Time Complexity.} Constructing the context-dependent DAG $\DAG{}$ for the realized context and performing the updates in \cref{eq:gamma-update-alt,eq:phi-update-alt} require $O(|\mathsf E_\v|)$ time and space. Sampling a path takes $O(M/\epsilon)$ time. Updating the shared coefficients via \cref{eq:main-iteration1,eq:main-iteration2} requires $O(|\overline{\mathsf E}|)$ time and space, where $|\overline{\mathsf E}|$ denotes the number of edges in the super DAG. Since $|\mathsf E_\v|\le |\overline{\mathsf E}|$ for all $\v\in\mathcal V$ (and in fact the two are of the same order), the overall per-round complexity of \cref{alg:weight-pushing-stoc-alt} is $O(|\overline{\mathsf E}|)$ in both time and space. \textit{In particular, the implementation is independent of $|\mathcal V|$, and thus applies even when the context space is infinite.}

\subsection{Proof of \cref{thm:regret-safe-bandit-unk-alt}}\label{apx:ssec:infinite-bandit}
We begin by presenting the efficient implementation of the learning algorithm in this setting.

\begin{algorithm}[!tbh]
\caption{\small No Budget Constraint (Bandit Setting, Unknown Distribution) -- Efficient Implementation}
\label{alg:weight-pushing-bandit-unk-alt}
\small{
\begin{algorithmic}[1]
\Require Learning rates $\eta_t>0$, $\delta\in(0, 1]$, and an edge path cover $\mathcal C$ of the super DAG. Define $\eta_0=1$. For all $e\in\overline{\mathsf{E}}$, define $\widehat{x}^0(e)=\widehat{y}^0(e)=0$.
\For{$t = 1, 2, \dots, T$}
    \State Observe an i.i.d. valuation vector sample $\v^t\sim\mathcal{D}$. Let $\v^t=\v'$.
    \State Sample $Z_t\sim\text{Unif}[0, 1]$.
    \State Construct $\xDAG{t}{\v'}$ and obtain edge probabilities $\phi^t_{\v'}(\cdot)$ following \cref{eq:gamma-update-alt} and \cref{eq:phi-update-alt}.
    \If{$Z_t\leq \delta$}
    \State Sample a path $\path^t$ from the edge path cover. 
    \Else\label{line:else-start}
    \State Define initial node $u=s$ and path $\path^t=s$. 
    \While{$u\neq d$}
    \State Sample $v$ with probability $\phi^t_{\v'}(u\to v)$.
    \State Append $v$ to the path $\path^t$; set $u\gets v$.
    \EndWhile\label{line:else-end-alt}
    \EndIf
    \State Map $\path^t=s\to (1, b_1, s_1)\to\dots\to (\maxbid, b_\maxbid, s_\maxbid)\to d$, and submit $\ibid^t=[b_1, \dots, b_\maxbid]$.\label{line:map-unk-end-alt}
    
    \State Update $\widehat{x}^t(e)$ and $\widehat{y}^t(e)$ for all $e\in\overline{\mathsf{E}}$ per \cref{eq:main-iteration1} and \cref{eq:main-iteration2} respectively, where $x^t(e)$ and $y^t(e)$ are defined in \cref{eq:xy-bandit-unk-alt}.
\EndFor
\end{algorithmic}}
\end{algorithm}

Observe that once the coefficients $\{x^t(e),y^t(e)\}_{e\in\overline{\mathsf E}}$ are available, the Hedge-style updates in \cref{alg:weight-pushing-bandit-unk-alt} follow the same structure as in the full-information implementation (cf.~\cref{alg:weight-pushing-stoc-alt}). Together with the uniform sampling over edge path cover, this shows that \cref{alg:weight-pushing-bandit-unk-alt} is equivalent to \cref{alg:weight-pushing-bandit-unk}. Thus the regret bound remains the same in this setting.

To obtain the space and time complexity, note that if the algorithm samples a path uniformly at random, this can be done in $O(|\overline{\mathsf{E}}|)$ time. If instead it performs exponential-weight updates, the edge marginals $\{p_{\v'}^t(e)\}_{e\in\overline{\mathsf{E}}}$ can be computed in $O(|\overline{\mathsf{E}}|)$ time and space (see \cref{apx:thm:regret-safe-bandit-unk}). The remaining steps mirror the full-information implementation and likewise take $O(|\overline{\mathsf{E}}|)$ time and space. Therefore, the overall per-round complexity of \cref{alg:weight-pushing-bandit-unk-alt} is $O(|\overline{\mathsf E}|)$ in both time and space. 


\end{document}